\documentclass[11pt]{article}
\usepackage{graphicx}
\usepackage{geometry}
\graphicspath{{Fig/}}

\usepackage{amsmath,amssymb,amsthm,mathtools}
\usepackage{float} 

\DeclareMathOperator*{\argmin}{arg\,min}

\DeclareMathOperator{\Tr}{Tr}
\newcommand{\R}{{\mathbb R}}

\newtheorem{theorem}{Theorem}
\newtheorem{proposition}[theorem]{Proposition}
\newtheorem{lemma}{Lemma}
\newtheorem{claim}{Claim}

\newtheorem{remark}{Remark}
\newtheorem{definition}{Definition}

\numberwithin{equation}{section}

\begin{document}

\title{Full-Model Optimality for Tunable Linear Generative Priors in Compressed Sensing}
\author{Zhaoming Li\thanks{Department of Mathematics, Northeastern University, Boston, MA 02115, USA. Email: \texttt{li.zhaom@northeastern.edu}} \and Paul Hand\thanks{Department of Mathematics and Khoury College of Computer Sciences, Northeastern University, Boston, MA 02115, USA.}}
\date{}
\maketitle

\begin{abstract}
Generative models have been studied experimentally and theoretically as priors for inverse problems such as compressed sensing. Recent work by Gunn et al. studied the use of generative priors with tunable complexity, where a family of generative priors with varying complexity is maintained and a specific complexity can be selected at inversion time.  They demonstrated that lower reconstruction errors can be experimentally attained for a variety of inverse problems by appropriately tuning the complexity of the generative prior.
In the present paper, we establish theory for compressed sensing in the setting
of a tunable family of linear generative priors naturally related through their singular value decompositions. 
We prove that in noiseless Gaussian compressed sensing, the full-dimensional linear prior attains the minimum expected reconstruction error over the entire family of linear priors.   Thus, in this idealized
linear noiseless setting, tuning to a lower-complexity prior does not
improve the expected reconstruction error.  This result is in contract to the behavior of denoising, where lower complexity priors attain lower reconstruction errors due to a standard bias-variance tradeoff.  
This result indicates that the experimental benefits of tunability in compressed sensing with neural network priors arises due to nonlinearities in the generative models.  
\end{abstract}
\section{Introduction}
Compressed sensing concerns the recovery of an unknown signal $x^\star \in \mathbb{R}^n$ from a limited number of linear measurements
$$y = A x^\star + \eta,$$ where $A \in \mathbb{R}^{m \times n}$ is a measurement operator (e.g., a linear projection, convolutional blur, or other
transformation)\cite{song2022solving} with $m \ll n$, and $\eta \in \mathbb{R}^m$ denotes measurement noise.
A wide range of practical problems can be formulated as a recovery of an unknown signal from the noisy linear measurements,
such as medical imaging\cite{lustig2008compressed,ye2019compressed, wunsch1996ocean}, signal processing\cite{orovic2016compressive}, and computational sensing\cite{laurenzis2018computational} and  single-pixel imaging. The above problem is ill-posed due to the compressive nature of the linear operator, meaning that there are infinitely many signals that fit the given measurements. Hence one needs additional assumptions to overcome the inherent ill-posedness.

To enable successful recovery in such severely underdetermined settings, one must impose structural assumptions on the unknown signal $x^\star$. Classical compressed sensing theory exploits low-dimensional structure,
such as sparsity~\cite{Cand_s_2007, tibshirani1996regression,candes2006stable,1614066}
in a known basis or low-rank~\cite{fazel2002matrix,candes2012exact,Recht_2010} structure for matrix recovery, to regularize the inverse problem. Although finding the sparsest solution to an underdetermined system of linear equations is NP-hard,
convex relaxations such as $\ell_1$ minimization can provably recover the true sparse signal $x^\star$
under suitable conditions on the measurement matrix, such as the Restricted Isometry Property 
or related Restricted
Eigenvalue Condition.
These models admit elegant theory and strong recovery guarantees.

More recently, there has been a surge of interest in the use of generative models as signal structure prior in inverse problems.
Motivated by the success of compressed sensing using generative models, a growing body of work studies recovery under the assumption that $x^\star$ lies in,
or is well-approximated by, the range of a generative model.
In this framework, the compressed sensing problem can be viewed as
$
    \text{finding } x \in \mathbb{R}^n 
    \ \text{such that} \ y = A x,
$
subject to the constraint that $x \in \mathrm{Range}(\mathcal G_k)$, where $\mathcal G_k : \mathbb{R}^k \to \mathbb{R}^n$ denotes a generative model that maps a low dimensional latent variable to a high dimensional signal. In practice, generative priors are typically employed through a two-stage workflow: 
\begin{enumerate}
    \item \textbf{Training}: A generative model is trained in an unsupervised learning on a representative dataset, independently of any specific downstream inverse problem; and 
    \item \textbf{Inversion}: At deployment time, an optimization problem is solved using the fixed, trained model to enforce consistency with the observed measurements.
\end{enumerate}
  This isolation enables the use of rich, data-driven signal models that can be developed independently of the downstream inverse problem and reused across multiple measurement settings. A commonality of much of the existing literature is that the complexity of the generative model is fixed at training time. While the notion of complexity may vary across model classes, for example, being implicitly determined by the network architectural choices in GANs\cite{bora2017compressed} or by the latent dimensionality in normalizing flows (NFs) \cite{durkan2019neural}, it cannot be adjusted once the model has been trained.
As a result, once a generative model has been trained, its complexity is fixed and cannot be adapted to the inverse problem at hand.

To address the issue of selecting model complexity at inversion time, recent work\cite{gunn2026latent} has introduced generative priors with vary complexity, in which a single trained model implicitly represents a family of signal classes indexed by a complexity parameter. We refer to such a family as tunable because its latent dimension can be
selected at inversion time.
This tunable idea provides a different perspective on how generative priors can be deployed. Rather than committing to a single model complexity at training time,
one can defer the choice of complexity until deployment, when the measurement operator,
noise level, and sampling rate are fully specified. Empirical studies\cite{gunn2026latent} demonstrate that reconstruction performance depends on this choice:
as the model complexity varies, the reconstruction error often exhibits a non-monotone, U-shaped behavior that depends on factors such as the noise level and the number of measurements, with an intermediate complexity often yielding the best recovery.These observations suggest that it is worthwhile to train a tunable generative model in advance, so that its complexity can be selected for the specific inverse problem encountered at test time.

Beyond empirical observations, recent work by Gunn et al. \cite{gunn2026latent} provides theoretical insight into tunable generative priors in a simplified setting.
Specifically, it analyzes the problem of denoising under additive Gaussian noise using a family of
linear generative models indexed by a complexity parameter.
In this setting, where the forward operator is the identity and all signal coordinates are observed,
the expected reconstruction error can be computed explicitly as a function of the model complexity.
The resulting characterization reveals a non-monotone, U-shaped dependence of the reconstruction
error on the complexity parameter, reflecting a fundamental trade-off between underfitting the signal
and overfitting the noise.
In particular, the analysis shows that the optimal model complexity depends on the noise level,
with higher noise favoring simpler models.

Our goal is to determine whether the benefits of tunability observed
experimentally in compressed sensing with generative priors already arise in
the simplest linear setting. This leads to a natural baseline question:
can a lower-complexity prior outperform the full prior when the generator is
linear and the target signal is itself drawn from the full model? The answer is not implied by the existing denoising theory. In denoising,
the benefit of an intermediate model complexity is explained by a trade-off
between underfitting the signal and overfitting the measurement noise. That
mechanism is absent in the noiseless compressed sensing problem considered
here. At the same time, because the measurement operator is underdetermined,
one might expect that discarding weak spectral directions could improve
reconstruction by restricting the feasible set. We show that this does not
occur under the fully averaged risk over the Gaussian signal and measurement
ensembles. For the tunable linear family
\(\{\mathcal G_k\}_{k=1}^n\) obtained by truncating the singular-value
decomposition of the full model \(G_n\), the expected reconstruction error
satisfies $E(n)\le E(k), 1\le k\le n.$
Thus, the full generative model attains the minimum expected reconstruction
error throughout the family. This conclusion holds for every admissible
decreasing singular-value spectrum, including spectra concentrated in a small
number of dominant directions. Hence, even when the signal distribution is
effectively low-dimensional, the bias introduced by spectral truncation is
not compensated by an improvement in reconstruction error.

Our result identifies noiseless Gaussian compressed sensing with linear
generative priors as a baseline setting in which tunability does not improve
the expected reconstruction error. In particular, linear spectral truncation
alone cannot account for the gains observed experimentally with more general
generative priors. The contrast suggests that such gains depend on mechanisms
absent from the present model, such as nonlinear generator geometry,
measurement noise, structured sensing operators, or properties of the
reconstruction algorithm. An important direction for future work is to
determine which of these mechanisms can rigorously produce a tunability
effect, particularly for nonlinear generators such as ReLU networks, noisy
compressed sensing, phase retrieval, and other nonlinear inverse problems.
\subsection{Methodology}
In this paper, we study tunable generative priors in the case of
linear generative models. A generative model induces a probability
distribution on the signal space by pushing forward a latent random
variable. In the linear setting, this means that signals are generated
according to
\begin{align}
     x = Gz,
    \qquad z\sim \mathcal N(0,I),\label{xGz}
\end{align}
where \(G\) is a matrix, we say that \(G\) maps a latent variable to the ambient signal
space.

We fix an invertible matrix \(G_n\in\mathbb R^{n\times n}\) and write its
singular value decomposition as
\begin{equation}
    G_n = U\Sigma V^\top,
    \qquad
    \Sigma=\operatorname{diag}(\sigma_1,\ldots,\sigma_n),
    \qquad
    \sigma_1\ge \cdots \ge \sigma_n>0,
\end{equation}

where $U, V \in \R^{n\times n}$ are orthogonal matrices. 

For each \(k\in [n]\), we define a \(k\)-dimensional
generative model as follows.  Let
\begin{equation}
     \mathcal{G}_k := U_k\Sigma_k\in\mathbb R^{n\times k},\label{priors}
\end{equation}
where \(U_k\) consists of the first \(k\) columns of \(U\), and \(
    \Sigma_k=\operatorname{diag}(\sigma_1,\ldots,\sigma_k).
\)
Each $\mathcal{G}_k$ induces a probability distribution by \eqref{xGz}, which can be verified to be $x \sim \mathcal{N}(0,\mathcal{G}_k\mathcal{G}_k^T)$. 

The family \(\{\mathcal{G}_k\}_{k=1}^n\) is therefore a nested hierarchy of linear
generative priors, ordered by model dimensionality. Observe that although $\mathcal{G}_n$ and $\mathrm{G}_n$ are different as matrices, they induce an identical distribution for $x$. 
Note that the reduced model \(\mathcal{G}_k\)
captures the \(k\) directions of largest variance of the full Gaussian
prior.

We consider a compressed sensing problem in which the signal is sampled
from the full-dimensional generative model,
\[
    x^\star = G_n z^\star,
    \qquad z^\star\sim\mathcal N(0,I_n),
\]
and the measurements are noiseless and Gaussian:
\begin{equation}
    y = A x^\star,
    \qquad
    A\in\mathbb R^{m\times n},
    \qquad
    A_{ij}\sim \mathcal N(0,1),
    \qquad
    m \ll n. \label{eq:Gaussian}
\end{equation}

\footnote{When \(k=m\), if \(AG_k\) is invertible,
the two optimization problems in \eqref{eq0} have the same
unique solution 
$\widehat z(m)=(A\mathcal G_m)^{-1}y$. For Gaussian $A$, $AG_k$ is full rank, this invertibility holds with probability one.}
\begin{align}
\widehat x(k)=\mathcal G_k\widehat z(k), \qquad \widehat z(k)
=
\begin{cases}
\displaystyle
\arg\min_{z_k\in\mathbb R^k}
\frac12\|z_k\|_2^2
\quad
\text{subject to}
\quad
 y = A\mathcal{G}_kz_k,
& k\ge m, \\[2ex]
\displaystyle
\arg\min_{z_k\in\mathbb R^k}
\frac12\|A\mathcal G_kz_k-y\|_2^2,
& k \le m.
\end{cases}
\label{eq0} 
\end{align}
The central quantity in this paper is the expected reconstruction error across the model complexities $k$:
\begin{align}
    \mathrm E(k)\coloneqq \mathrm E_A\mathrm E_{x^{\star}\thicksim \mathcal{G}_n}\left\| \hat{x}(k) -x^{\star} \right\|_{2}^{2} = \mathrm E_A\mathrm E_{z^\star\thicksim \mathcal{N}(0,I_n)}\left\| 
\mathcal{G}_k\hat{z}(k)-\mathcal{G}_nz^\star\right\|_{2}^{2} \label{eq:recon err}
\end{align}
We emphasize that the signal is sampled from the full-complexity model
\(G_n\), whereas reconstruction imaging is performed using a potentially
lower-complexity prior \(G_k\).
This quantity captures how reconstruction quality varies with the chosen model dimension 
k, and in particular whether a generative prior of lower latent dimension can outperform the full prior.
\subsection{Main results}\label{sec:standing}
Our main results concern the expected reconstruction error for noiseless compressed sensing under a nested family of linear generative models, as
specified in \eqref{eq:Gaussian}.
The main theorem shows that, in the noiseless
linear setting, choosing a lower-dimensional model at inversion time cannot
outperform the full model.
\begin{theorem}[Noiseless optimality of the full model]
\label{thm:noiseless-optimality}
Fix an invertible matrix \(G_n\in\mathbb R^{n\times n}\) with singular value
decomposition
\[
    G_n=U\Sigma V^\top,
    \ \Sigma=\operatorname{diag}(\sigma_1,\ldots,\sigma_n),
    \
\sigma_1\ge\cdots\ge\sigma_n>0,
\]
where \(U,V\in\mathbb R^{n\times n}\) are orthogonal.
For each \(k\in[n]\), define the rank-\(k\) linear generative model
\[
    \mathcal G_k:=U_k\Sigma_k\in\mathbb R^{n\times k},
    \qquad
    \Sigma_k:=\operatorname{diag}(\sigma_1,\ldots,\sigma_k),
\]
where \(U_k\) consists of the first \(k\) columns of \(U\).
Let \(A\in\mathbb R^{m\times n}\) have i.i.d.\
\(\mathcal N(0,1)\) entries, where \(2\le m\le n\).
Let \(z^\star\sim\mathcal N(0,I_n)\) be independent of \(A\), and define
\[
    x^\star:=G_nz^\star,
    \qquad
    y:=Ax^\star .
\]
For each \(k\in[n]\), let \(\widehat x(k)\) be the reconstruction signal obtained
from the model \(\mathcal G_k\) according to the reconstruction rule
\eqref{eq0}, and let \(E(k)\) denote the expected reconstruction error
defined in \eqref{eq:recon err}.\footnote{The quantity \(E(k)\) is
interpreted as an extended-real-valued expectation. In particular,
\(E(m-1)=E(m)=+\infty\); see
Proposition~\ref{prop:boundary-divergence}.}
Then
\[
    E(n)\le E(k)
    \qquad\text{for every }k\in[n].
\]
\end{theorem}

Theorem~\ref{thm:noiseless-optimality} establishes that noiseless compressed
sensing with linear generative priors does not exhibit a tunability effect.
Specifically, if the underlying signal is sampled from an invertible linear
model, then using any corresponding lower-rank linear model as a signal prior
results in an expected reconstruction error no smaller than that obtained
using the full model. This conclusion holds for every admissible decreasing
singular value spectrum of
the full-dimensional model. In particular, it applies even when the signal model is effectively low-dimensional.

The proof of Theorem~\ref{thm:noiseless-optimality} is divided into three
regimes according to the value of the latent dimension \(k\) relative to
the measurement dimension \(m\).
For \(m+1\le k\le n\), the latent reconstruction problem is
underdetermined. In this regime, we derive an exact expression for the
expected reconstruction error and show that \(E(k)\) is finite and
monotone decreasing in \(k\). 
For the two critical dimensions \(k=m-1\) and \(k=m\), the exact error
formulas involve first inverse moments of Gaussian Gram matrices that are
infinite. Consequently,
$E(m-1)=E(m)=+\infty.$
For \(1\le k\le m-2\), the reconstruction problem is strictly
overdetermined. We derive an exact formula for \(E(k)\) and compare it
with an upper bound for the full-model error \(E(n)\), obtaining $ E(k)\ge E(n).$
Combining the conclusions from the three regimes shows that $E(n)\le E(k)$ for every $ k\in[n],$
and proves Theorem~\ref{thm:noiseless-optimality}.

We first record the monotonicity result in the underdetermined regime. 
\begin{theorem}\label{Theorem1}
    Under the assumptions of Theorem~\ref{thm:noiseless-optimality}, the expected reconstruction error \(E(k)\) is finite, and \(E(k)\) is monotone decreasing for every $k\in\{m+1,\dots,n\}$.

\end{theorem}
In the critical regime, the expected reconstruction error is infinite at both $k=m-1$ and $k=m$. This happens because the relevant Gaussian matrices can be arbitrarily close to singular, causing the inverse terms in the reconstruction error to have nonintegrable tails. Thus, although the estimators are well-defined almost surely, their expected reconstruction errors satisfy $E(m-1)=E(m)=+\infty.$

\begin{proposition} \label{prop:boundary-divergence}
    Under the assumptions of Theorem~\ref{thm:noiseless-optimality}, assume
\(2\le m<n\). Then the expected reconstruction error diverges:
\[
E(m-1)=+\infty,
\qquad
E(m)=+\infty.
\]
\end{proposition}

Combining Theorem~\ref{Theorem1} with Proposition~\ref{prop:boundary-divergence}, we can rule out all model dimensions at and above the measurement threshold, except for the full model. Indeed, Proposition~\ref{prop:boundary-divergence} gives $E(m)=+\infty,$
while Theorem~\ref{Theorem1} implies
$E(n)\le E(k)
\ 
\text{for every } m+1\le k\le n.$
Therefore, 
$E(n)\le E(k)
\
\text{for every } m \le k\le n.$
No model dimension \(m\le k \le n\) can outperform the full model. Moreover, Proposition~\ref{prop:boundary-divergence} also gives
$E(m-1)=+\infty,$
which rules out the endpoint \(k=m-1\) of the overdetermined regime. Hence the only lower-dimensional models that remain to be compared with the full model are $1\le k\le m-2.$

Our next result the strictly overdetermined regime \(1\le k\le m-2\). In this regiem, the reconstruction error admits an exact closed-form expression, which
will later allow us to compare the lower-dimensional model with the full
model.
\begin{lemma}
\label{lemma:closed-form-k-less-m}
Under the same model assumptions as Theorem~\ref{thm:noiseless-optimality},
for every integer \(k\in\{1,\ldots,m-2\}\),
error of the MAP estimator using the model \(\mathcal G_k\) satisfies
\[
    E(k) =
    \left(1+\frac{k}{m-k-1}\right)
    \sum_{i=k+1}^{n}\sigma_i^2 \ge E(n).
\]
\end{lemma}
Lemma~\ref{lemma:closed-form-k-less-m} gives an exact formula for
\(E(k)\) when \(1\le k\le m-2\) and shows that $E(k)\ge E(n).$
Hence the full model is optimal throughout the strictly overdetermined
regime.

Our analysis directly studies reconstruction error in expectation, rather than
establishing a high-probability or asymptotic upper bound. This differs from
much of the standard non-asymptotic compressed sensing literature, where recovery guarantees are often
stated as deterministic consequences of conditions such as the restricted
isometry property or null-space property \cite{candes2005decoding,foucart2013mathematical}, together with high-probability results
showing that random measurement matrices satisfy these conditions \cite{candes2005decoding,candes2006robust,foucart2013mathematical}. By contrast,
in our linear generative setting we derive an exact decomposition of the
expected reconstruction error under the Gaussian measurement ensemble. These expressions make the dependence of
\(E(k)\) on the latent dimension explicit: for \(k<m\), they give a
closed-form formula, while for \(k\ge m\), they yield monotonicity of
the risk. Combining the two regimes shows that the full prior \(G_n\)
minimizes the expected reconstruction error over the entire tunable
linear family. Hence the noiseless linear model serves as a baseline
case in which tunability is available at the level of model selection,
but does not improve the fully expected reconstruction error.
\subsection{Experimental validation}
\label{sec:experimental-validation}

We numerically illustrate Theorem~\ref{thm:noiseless-optimality} for several
representative singular-value spectra. We set \(n=100\) and \(m=20\).
By the diagonal reduction in Lemma~\ref{lem:diagonal-reduction}, it suffices
to take $G_n=\Sigma=\operatorname{diag}(\sigma_1,\ldots,\sigma_n).$
We consider the following five spectra:
\begin{itemize}
    \item Flat spectrum:
    \[
        \sigma_i=1,
        \qquad 1\le i\le n.
    \]

    \item Step spectrum:
    \[
        \sigma_i=
        \begin{cases}
            1, & i\le k_0,\\
            0.1, & i>k_0,
        \end{cases}
        \qquad k_0=\lfloor m/7 \rfloor=2.
    \]

    \item Power-law spectrum:
    \[
        \sigma_i=i^{-\alpha},
        \qquad \alpha=2.
    \]
    \item Exponential spectrum:
    \[
        \sigma_i=e^{-\beta(i-1)},
        \qquad \beta=0.1.
    \]

    \item Linear spectrum:
    \[
    \sigma_i=
        \begin{cases}
            1, & n=1,\\
            1-\varepsilon\frac{i-1}{n-1},\ \varepsilon =1, & o therwise.
        \end{cases}
    \]
\end{itemize}

Rather than solving the reconstruction problem separately for each realization,
we evaluate the expected reconstruction error using the analytical expressions
derived above. For \(1\le k\le m-2\), we compute \(E(k)\) directly from
\eqref{eq:E(k)over}. For \(m+1\le k\le n\), we evaluate the expression in
\eqref{eq:E(k)un} and approximate its remaining expectation over the measurement
matrix using \(T=100\) independent matrices
\(A\in\mathbb R^{m\times n}\) with i.i.d.\
\(\mathcal N(0,1)\) entries. The expectation over the Gaussian latent variable
has already been evaluated analytically in these formulas. The same samples of
\(A\) are used to estimate \(E(n)\). We omit \(k=m-1\) and \(k=m\), since $ E(m-1)=E(m)=+\infty.$
Figure~\ref{fig:exp} displays the singular-value spectra in the top row and the
corresponding excess errors
$  E(k)-E(n)$
in the bottom row, for
$  k\in\{1,\ldots,m-2\}\cup\{m+1,\ldots,n\}.$
A positive excess error means that the model \(\mathcal G_k\) has larger expected
reconstruction error than the full model \(\mathcal G_n\).

\begin{figure}[H]
    \centering
    \includegraphics[width=1.3\linewidth]{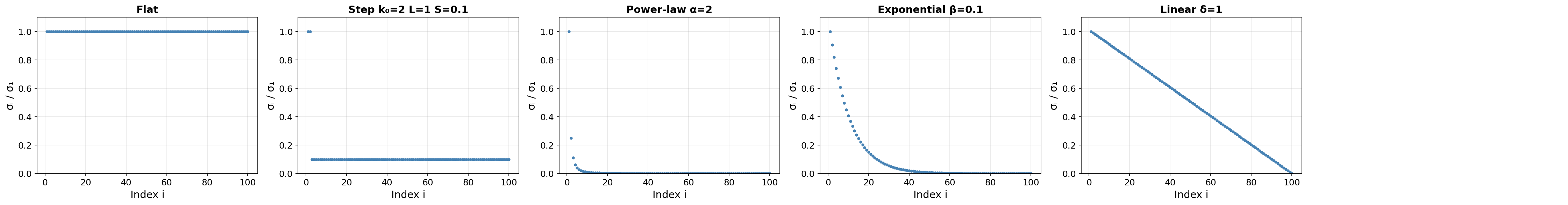}

    \vspace{0.5em}

\includegraphics[width=1.3\linewidth]{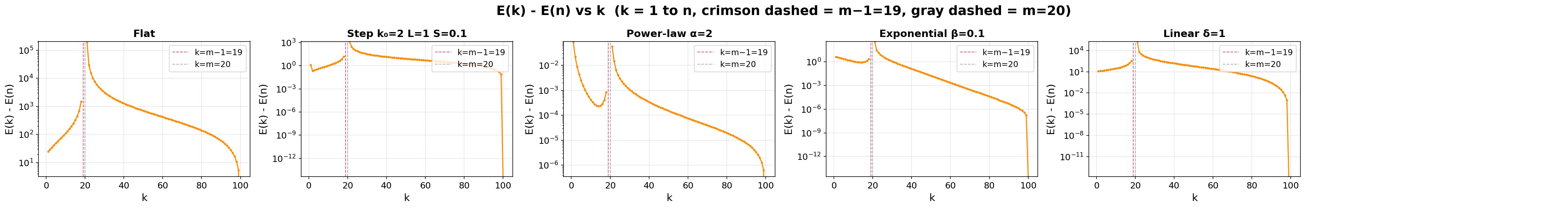}

    \caption{
    Singular-value spectra (top) and corresponding excess reconstruction errors
    \(E(k)-E(n)\) (bottom), with \(n=100\) and \(m=20\).
    The dimensions \(k=m-1\) and \(k=m\), indicated by vertical dashed lines,
    are omitted because their expected reconstruction errors are infinite.
    }
    \textbf{Alt text:}
    Two rows of plots showing five singular-value spectra (top) and the corresponding excess reconstruction errors $E(k)-E(n)$ versus latent dimension $k$ (bottom). Vertical dashed lines mark $k=m-1$ and $k=m$.
    \label{fig:exp}
\end{figure}

Across all five spectra, $E(k)-E(n)$ is nonnegative, in agreement with
Theorem~\ref{thm:noiseless-optimality}. For some spectra, \(E(k)\) is
non-monotone and may exhibit a U-shaped profile in the strictly overdetermined
regime \(1\le k\le m-2\). Nevertheless, these local minima remain above
\(E(n)\). In the underdetermined regime \(m+1\le k\le n\), the error decreases
monotonically toward the full-model error. Thus, although \(E(k)\) can vary non-monotonically with the dimension over overdetermined regime, in all of the experiments reported above, the minimum reconstruction error is attained by the full model \(k=n\).

\subsection{Discussion}
The results of this paper provide a complete characterization of the effect of model dimension in the noiseless linear Gaussian setting. For \(1\leq k\leq m-2\), the closed-form expression for \(E(k)\) yields the direct comparison \(E(k)\geq E(n)\). At the two critical dimensions \(k=m-1\) and \(k=m\), the expected reconstruction error diverges, while for \(k\geq m+1\), \(E(k)\) is finite and monotonically decreasing. Combining these three regimes gives $E(n)\leq E(k),
1\leq k\leq n.$
Thus, although the dependence of \(E(k)\) on the dimension may be locally non-monotone, the full model is globally optimal within the nested family of linear generative priors considered here.

An important consequence is that effective low-dimensionality of the signal distribution is not, by itself, sufficient to produce a tunability benefit. The full-model optimality result holds for every admissible decreasing singular-value spectrum, including spectra in which most of the signal variance is concentrated in a small number of leading directions. Hence, even when the full model is effectively low-dimensional, truncating the weak spectral directions cannot reduce the fully averaged reconstruction error below that of the full prior. In this setting, the bias introduced by restricting the signal model is never compensated by a sufficient reduction in reconstruction error.

These conclusions also clarify the scope of the linear theory. They show that tunability observed in more general inverse problems cannot be explained solely by spectral concentration or by the underdetermined nature of the measurement system. Rather, such behavior must arise from mechanisms that are absent from the present noiseless linear Gaussian model. In particular, a tunability benefit may become possible when the generative model is nonlinear, when the measurement map is nonlinear, or when both forms of nonlinearity are present. Other departures from the present setting, including measurement noise or algorithm-dependent effects, may also alter the dependence of reconstruction error on model complexity. Determining which of these mechanisms is responsible for tunability, and how it interacts with the choice of latent dimension, remains an important direction for future work.
\subsection{Organization of the paper}
The remainder of the paper is organized as follows. 
Subsection~\ref{sec:notations} introduces the notation used throughout the paper.
In Section~\ref{sec:Architecture of the Proof}, we present the architecture of
the proof of Theorem~\ref{thm:noiseless-optimality}. The proof is divided into
three regimes. Section \ref{secun}, we consider  underdetermined regime
\(m+1\le k\le n\), we prove that \(E(k)\) is monotone decreasing in \(k\).
In section \ref{secov}, it is a strictly overdetermined regime \(1\le k\le m-2\),
we derive a closed-form expression for \(E(k)\) and compare it with the
full-model error. Finally, in Section \ref{sec:boundary-divergence} we consider the boundary dimensions \(k=m-1\) and \(k=m\), we show that
the expected reconstruction error diverges.
\subsection{Notation}
\label{sec:notations}We introduce notation that will be used throughout the proofs of the
monotonicity result for \(k\ge m\), Theorem~\ref{Theorem1}, and the
closed-form risk formula for \(k<m\), Lemma~\ref{lemma:closed-form-k-less-m}.
This notation is also used in the proof of
Theorem~\ref{thm:noiseless-optimality}, which combines the two regimes
to show that the full model minimizes the expected reconstruction error
over the entire tunable linear family.
Let $[n] = \{1,\dots,n\}$. For $A \in \R^{m \times n}$, $k \in [n]$, we will write $A_k$ as the first $k$ columns of $A$, $\widetilde{A_k}$ as the last $n-k$ columns of $A$. Thus $A =\Big[A_k,\widetilde{A_k}\Big]$. 
Let $a_i \in \R^m$ be the $i$-th column of $A$. In the case of a diagonal $\mathrm{G}_k$ in Theorem \ref{Theorem1} we can write $\mathrm{G}_n = \begin{pmatrix}
   G_k & 0\\
   0 & \widetilde{G_k}
\end{pmatrix}$, where $\mathrm{G}_k=diag(\sigma_1,\dots,\sigma_k)$, $\widetilde{G_k}=diag(\sigma_{k+1},\dots,\sigma_n)$.  For any  $M \in \R^{n \times k}$ with $k \leq n$, suppose  $M$ is full rank; let $P_{M^T} = M^T\Big(MM^T\Big)^{-1}M$ be the orthogonal projection on Range($\mathrm{M})$. Abusing notation, when we write $(M_1-M_2)_{ii}$, it means $(M_1)_{ii}-(M_2)_{ii}$.
\section{Proof of the main theorem}\label{sec:Architecture of the Proof}
The proof of Theorem~\ref{thm:noiseless-optimality} has three components. 
First, we reduce the analysis to diagonal generators in Section \ref{sec:diag}. This reduction
stated in Lemma~\ref{lem:diagonal-reduction}, follows from the rotational
invariance of the Gaussian measurement operator and shows that \(E(k)\)
depends only on the singular values of \(G_n\). Consequently, it suffices to
study $G_n=\Sigma=\operatorname{diag}(\sigma_1,\ldots,\sigma_n).$
Second, we analyze the expected reconstruction error across the different
model complexity regimes, and treating the two dimensions adjacent to the
measurement threshold separately.
In Section~\ref{secun}, for \(k\ge m \), we derive an exact
decomposition that yields the monotonicity of \(E(k)\), 
this monotonicity result follows from an explicit analysis of the projection structure 
induced by the measurement operator. Hence the full model \(G_n\) minimizes
the expected reconstruction error within the regime \(k\ge m\). 
In
Section~\ref{secov}, for \(k<m\), we derive the corresponding decomposition
and find the closed formula of expected reconstruction error. This formula makes explicit how the reconstruction error depends on the
latent dimension and the singular-value spectrum in the regime where
the number of measurements exceeds the latent dimension. Finally, in Section~\ref{sec:boundary-divergence}, we complete the proof by
handling the two remaining critical dimensions \(k=m-1\) and \(k=m\), where the
corresponding inverse Gaussian terms have non-integrable tails.
\subsection{Reduction to diagonal generators}\label{sec:diag}
\begin{lemma}\label{lem:diagonal-reduction}
Let \(G_n=U\Sigma V^\top\) and let \(\mathcal G_k=U_k\Sigma_k\). Under
Gaussian measurements, the expected reconstruction error \(E(k)\) depends
only on the singular values \((\sigma_1,\ldots,\sigma_n)\). Consequently, it
suffices to prove Theorem~\ref{thm:noiseless-optimality} in the diagonal case
\[   G_n=\Sigma=\operatorname{diag}(\sigma_1,\ldots,\sigma_n).
\]
\end{lemma}
\allowdisplaybreaks{
\begin{proof}[Proof of Lemma \ref{lem:diagonal-reduction}.]
Without loss of generality, it suffices to consider the case when $\mathrm{G}_n$ is diagonal, that is $U=V=I$. To see this, let $\mathrm{G}_n = U\Sigma V^T $ and $\mathcal{G}_k = U_k\Sigma_k$.
    First we define \begin{align*}
   \hat{z}(k,A,\mathcal{G}_k) = 
\begin{cases} 
\underset{z_{k} \in \R^k}{\argmin}\ \frac{1}{2}\| z_k \|_2^2 \quad \text{s.t.} \quad y = A\mathcal{G}_kz_k, & \text{for } k \geq m ,  \\
\underset{z_{k} \in \R^k}{\argmin} \ \frac{1}{2} \left\|y- A\mathcal{G}_k z_k  \right\|_2^2, & \text{for }  k < m .
\end{cases} \ 
\end{align*}
Observe that $\hat{x}(k)= \mathcal{G}_k\hat{z}(k,A,U_k\Sigma_k) =U_k\Sigma_k\hat{z}(k,A,U_k\Sigma_k).$
The expected reconstruction error satisfies: 
\begin{align}
    \mathrm E(k) &:= \mathrm E_A\mathrm E_{x^\star\thicksim \mathcal{G}_n}\left\|\hat{x}(k) -x^\star\right\|_{2}^{2} \nonumber\\
    & = \mathrm E_A\mathrm E_{z^\star\thicksim \mathcal{N}(0,I_n)}\left\|\mathcal{G}_k\hat{z}(k,A,U_k\Sigma_k)-\mathcal{G}_n z^\star\right\|_{2}^{2}\nonumber\\
    & = \mathrm E_A\mathrm E_{z^\star}\left\|U_k\Sigma_k\hat{z}(k,A,U_k\Sigma_k) -U\Sigma z^\star\right\|_{2}^{2}\nonumber\\
    & = \mathrm E_A\mathrm E_{z^\star}\left\|U\binom{\Sigma_k}{0}\hat{z}(k,A,U_k\Sigma_k) -U\Sigma z^\star\right\|_{2}^{2} \nonumber \\
    & = \mathrm \mathrm E_A\mathrm E_{z^{\star}}\left\|\binom{\Sigma_k}{0}\hat{z}(k,A,U_k\Sigma_k) -\Sigma z^\star\right\|_{2}^{2}, \label{eq1}
\end{align}
where the second-to-last equation is given by the unitary invariance of the $\ell_2$ norm.  It remains to show that $\mathrm E(k)$ does not depend on $U$.\\
Observe $\hat{z}(k,A,U_k\Sigma_k) = \hat{z}(k,AU,\binom{\Sigma_k}{0}).$ Thus,
\begin{align*}
    \mathrm  E(k) &=\mathrm E_A\mathrm E_{z^{\star}}\left\|\binom{\Sigma_k}{0}\hat{z}(k,A,U_k\Sigma_k) -\Sigma z^\star\right\|_{2}^{2}\\
    & = \mathrm E_A\mathrm E_{z^\star}\left\|\binom{\Sigma_k}{0}\hat{z}(k,AU,\binom{\Sigma_k}{0}) -\Sigma z^\star\right\|_{2}^{2}\\
    & = \mathrm E_A\mathrm E_{z^\star}\left\|\binom{\Sigma_k}{0}\hat{z}(k,A,\binom{\Sigma_k}{0}) -\Sigma z^\star\right\|_{2}^{2}.\\
\end{align*}
where the last equality follows because $A$ and $AU$ are equal in distribution due to the rotational invariance of the Gaussian. This concludes the proof that without loss of generality, we can take 
$U=V=I$, which is the case of a diagonal $\mathrm{G}_n = diag(\sigma_1,\dots,\sigma_n)$.
By Theorem \ref{Theorem1} in Section \ref{sec:standing}, $\mathrm E(k)$ decreases monotonically over $ m\leq k \leq n$.
By Lemma \ref{lemma:closed-form-k-less-m} in Section \ref{sec:standing} , we show that there exist no$k$ over $ 1\leq k \leq m$ the reconstruction error $\mathrm E(k)$ beats the $\mathrm E(n)$.
\end{proof}}
\subsection{Proof of Main Result}
\begin{proof}[Proof of Theorem \ref{thm:noiseless-optimality}.]
Fix a spectrum $\sigma_1\ge\cdots\ge\sigma_n>0$.
We prove that \(E(n)\le E(k)\) for every \(k\in \{1,\dots,n\}\) by splitting the range of \(k\) into three regimes around the measurement threshold $m$. Let
\[
C(k):=\Bigl(1+\frac{k}{m-k-1}\Bigr)\sum_{i=k+1}^{n}\sigma_i^2 .
\]

\noindent\textbf{Regime 1: $m<k\le n$.}
Here $A\mathcal G_k\in\mathbb R^{m\times k}$ is wide, and by
Theorem \ref{Theorem1}, $E(k)$ is monotonic decreasing on
$\{m,\dots,n\}$. Hence its minimum over this range is attained at the right endpoint
$k=n$, so
\[
E(k)\;\ge\;E(n)\qquad\text{for all }m<k\le n.
\]

\noindent\textbf{Regime 2: $1\le k\le m-2$.}
\[
E(k)=C(k)=\Bigl(1+\frac{k}{m-k-1}\Bigr)\sum_{i=k+1}^{n}\sigma_i^2 .
\]
Moreover, by Lemma~\ref{lemma:closed-form-k-less-m}, it provides a bound on the full-model error for every admissible internal rank $k$ satisfying $m-k \ge 2$,
\[
E(n)\;\le\;\Bigl(1+\frac{k}{m-k-1}\Bigr)\sum_{i=k+1}^{n}\sigma_i^2\;=\;C(k).
\]
Combining the two displays gives $E(n)\le C(k)=E(k)$ for all $1\le k\le m-2$.

\noindent\textbf{Regime 3: $k\in\{m-1,\,m\}$.}
These two boundary dimensions are handled by Proposition~\ref{prop:boundary-divergence}, which states that
\[
E(m-1)=+\infty,
\
E(m)=+\infty.
\]
Since \(E(n)\) is finite by Theorem~\ref{Theorem1}, we immediately obtain
\[
E(n)\le E(k),
\qquad
k\in{m-1,m}.
\] 

\noindent The three regimes
$1\le k\le m-2,
\
k\in\{m-1,m\},
\
m+1\le k\le n$ exhaust $\{1,\dots,n\}$, and in each regime we obtained
$E(n)\le E(k)$. Since the spectrum was arbitrary, the inequality holds for every
strictly decreasing admissible spectrum, which completes the proof.
\end{proof}
\subsection{Main Result for Underdetermined System}\label{secun}
Recall without loss of generality, we have assumed that $\mathrm{G}_n = diag(\sigma_1,\dots,\sigma_n)$, then $\mathcal{G}_n = \mathrm{G}_n$. In this case $\mathcal{G}_k = \binom{\mathrm{G}_k}{0} \in \R^{n \times k}$ where $\mathrm{G}_k= diag(\sigma_1,\dots,\sigma_k)$. For $k \geq m$, 
\begin{align}
&\hat{x}(k)= \mathcal{G}_k\hat{z}(k) \nonumber \\
    &\hat{z}(k)=\argmin_{z_k \in \R^k}\ \frac{1}{2} \| z_k \|_2^2 \quad \text{s.t.} \quad y = A\mathcal{G}_k z_k. \label{eq2}
\end{align}
\begin{lemma}\label{Under_E(k)}
 Suppose $x^\star \thicksim \mathcal{G}_n$, $\mathcal{G}_n = diag(\sigma_1,\dots,\sigma_n)$ has singular values $\sigma_1 \geq \sigma_2 \geq \dots \geq \sigma_n >0 $. Let $G_k = diag(\sigma_1,\dots,\sigma_k)$. Let $A \in \R ^{m\times n}$, $A_{ij}\stackrel{i.i.d}{\sim} \mathcal{N}(0,1)$. For fixed $A$, for $k \in [m,n]$,  and the MAP estimator from $\eqref{eq2}$ obeys 
\begin{align}
    \mathrm E_{x^\star \sim \mathcal{G}_n} \big\| \hat{x}(k) - x^\star \big\|_2^2 = \Big\| G_k P_{G_k^T A_k^T} - G_k \Big\|_F^2 + \big\| \widetilde{G}_k \big\|_F^2 \nonumber + \Big\| G_k (A_k G_k)^\dagger \widetilde{A}_k \widetilde{G}_k \Big\|_F^2,
\end{align}
furthermore,
\begin{align*}
    \mathrm E(k) &=\mathrm E_A \mathrm E_{x^\star \sim \mathcal{G}_n} \big\| \hat{x}(k) - x^\star \big\|_2^2 \\
    &= \mathrm E_{A_k} \left( \| G_n \|_F^2 - \Big\langle G_k^2, P_{G_k^T A_k^T} \Big\rangle + \Big\| G_k (A_k G_k)^\dagger \Big\|_F^2 \sum_{i = k+1}^n \sigma_i^2 \right)
\end{align*}
where $P_{G_k^T A_k^T} = G_k^T A_k^T (A_k G_k G_k^T A_k^T)^{-1} A_k G_k$.
\end{lemma}
\begin{proof}[Proof of Lemma \ref{Under_E(k)}.]
    By \eqref{eq2}, the constraint that $y = A\mathcal{G}_k\hat{z}(k)$. We write $z^\star = \binom{z^\star_k}{\widetilde{z}^\star_k}$, where $z^\star_k \in \mathbb{R}^k, \ \widetilde{z}^\star_k \in \mathbb{R}^{n-k}$. As $A\mathcal{G}_k = A_kG_k$ has rank $m$ with probability 1, the solution to \eqref{eq2} is 
\(
\hat{z}(k) = (A_kG_k)^{\dagger}AG_nz^\star,
\)
where $(A_kG_k)^{\dagger} = G_k^TA_k^T(A_kG_kG_k^TA_k^T)^{-1}$. \\
The expected reconstruction error satisfies: 
{\allowdisplaybreaks
\begin{align}
\mathrm E_{x^\star \sim \mathcal{G}_n} \left\| \hat{x}(k) - x^\star \right\|_{2}^{2}
& = \mathrm E_{z^\star \sim \mathcal{N}(0,I)} \left\| \mathcal{G}_k \hat{z}(k) - \mathcal{G}_n z^\star \right\|_{2}^{2} \nonumber \\
& =\mathrm  E_{z^\star} \left\| \binom{G_k}{0} \hat{z}(k) - G_n z^\star \right\|_{2}^{2} \nonumber \\
& =\mathrm  E_{z^\star}\left\| \binom{G_k}{0} \hat{z}(k) - \begin{pmatrix} G_k & 0 \\ 0 & \widetilde{G_k} \end{pmatrix} z^\star \right\|_{2}^{2} \nonumber\\
& = \mathrm E_{z^\star} \left\| \begin{pmatrix} G_k \hat{z}(k) \\ 0 \end{pmatrix} - \begin{pmatrix} G_k z^\star_{k} \\ \widetilde{G_k} \widetilde{z^\star_{k}} \end{pmatrix} \right\|_2^2 \nonumber\\
& =\mathrm  E_{z^\star} \left\| G_k \hat{z}(k) - G_k z^\star_k \right\|_2^2 + E_{z^\star} \left\| \widetilde{G_k} \widetilde{z^\star_{k}} \right\|_2^2 \nonumber \\
& =\mathrm  E_{z^\star} \left\| G_k (A_k G_k)^{\dagger} A G_n z^\star - G_k z^\star_k \right\|_2^2 + \| \widetilde{G_k} \|_F^2  \label{eq3} \\
& =\mathrm  E_{z^\star} \left\| G_k (A_k G_k)^{\dagger} (A_k G_k z^\star_k + \widetilde{A_k} \widetilde{G_k} \widetilde{z^\star_{k}}) - G_k z^\star_k \right\|_2^2 + \| \widetilde{G_k} \|_F^2 \nonumber\\
& =\mathrm  E_{z^\star} \left\| (G_k (A_k G_k)^{\dagger} A_k - I_k) G_k z^\star_{k} + G_k (A_k G_k)^{\dagger} \widetilde{A_k} \widetilde{G_k} \widetilde{z^\star_{k}} \right\|_2^2 + \| \widetilde{G_k} \|_F^2 \nonumber\\
& =\mathrm  E_{z^\star} \left\| (G_k (A_k G_k)^{\dagger} A_k - I_k) G_k z^\star_{k} \right\|_2^2 \nonumber \\
& \quad + 2\mathrm  E_{z^\star} \Big\langle \big(G_k (A_k G_k)^{\dagger} A_k - I_k\big) G_k z^\star_{k}, G_k (A_k G_k)^{\dagger} \widetilde{A_k} \widetilde{G_k} \widetilde{z^\star_{k}} \Big\rangle \nonumber\\
& \quad +\mathrm  E_{z^\star} \left\| G_k (A_k G_k)^{\dagger} \widetilde{A_k} \widetilde{G_k} \widetilde{z^\star_{k}} \right\|_2^2 + \| \widetilde{G_k} \|_F^2\nonumber \\
& = \Big\| G_k (A_k G_k)^{\dagger} A_k G_k - G_k\Big \|_F^2 + \Big\| G_k (A_k G_k)^{\dagger} \widetilde{A_k} \widetilde{G_k} \Big\|_F^2 + \| \widetilde{G_k} \|_F^2 \label{eq4} \\
& = \Big\| G_k P_{G_k^T A_k^T} - G_k \Big\|_F^2 + \big\| \widetilde{G_k} \big\|_F^2 + \Big\| G_k (A_k G_k)^{\dagger} \widetilde{A_k} \widetilde{G_k} \Big\|_F^2  \label{eq5}
\end{align}
$\eqref{eq3}$ follows by Lemma \ref{Lemma10}, and $\eqref{eq4}$ follows by Lemma \ref{Lemma10} and the fact that $z_k^\star $, and $\widetilde{z_k^\star }$ are independent with zero mean. 
By Lemma \ref{lemma2},
we get:
\begin{align}
\mathbb{E}(k) &= \|G_n\|_F^2 - \mathbb{E}_{A} \left\langle G_k^2, P_{G_k^T A_k^T} \right\rangle + \mathbb{E}_{A} \left\| G_k (A_k G_k)^{\dagger} \widetilde{A}_k \widetilde{G}_k \right\|_F^2 \nonumber \\
&= \|G_n\|_F^2 - \mathbb{E}_{A_k} \left\langle G_k^2, P_{G_k^T A_k^T} \right\rangle + \mathbb{E}_{A_k} \left\| G_k (A_k G_k)^{\dagger} \right\|_F^2 \sum_{i=k+1}^n \sigma_i^2. \label{eq:E(k)un}
\end{align}
}
\end{proof}
We now return to the proof of Theorem~\ref{Theorem1}. 
After deriving the exact expression for \(E(k)\) in the regime \(k\ge m\), 
we verify that this expression is monotone decreasing in the selected latent 
dimension \(k\).
\begin{proof}[Proof of Theorem \ref{Theorem1}.]

For fixed $A$, let $R_k \coloneq \mathrm  E_{x^\star\thicksim \mathcal{G}_n}\big\| \hat{x}(k) -x^\star \big\|_{2}^{2}$, where $\hat{x}(k)$ is given by \eqref{eq3}. 
By definition,
$\mathrm E(k)=\mathrm E_A\big[R_k\big]$. 
Thus, for $m <k \leq n$,
{\allowdisplaybreaks
\begin{align*}
    \mathrm E(k) - \mathrm E(k-1) =\mathrm E_A\big[R_k\big]-\mathrm E_A\big[R_{k-1}\big]
=\mathrm E_A\big[R_k-R_{k-1}\big]
     =\mathrm E_{A_{k-1}}\mathrm E_{\widetilde{A_{k-1}}}\big[R_k-R_{k-1}\big].  
\end{align*}}
In order to show $E(k) - E(k-1) \leq 0$, we will show that for all full rank $A_{k-1}$, 
\begin{align}
  \mathrm  E_{\widetilde{A_{k-1}}}\big[R_k-R_{k-1}\big] \leq 0. \label{eq6}
\end{align}
Since $\widetilde{A_{k-1}} =[a_k, \widetilde{A_{k}}]$, let $f_k \coloneq \mathrm E_{\widetilde{A_{k}}}\big[R_k\big]$. Therefore, we have
\begin{align}
   \mathrm  E_{\widetilde{A_{k-1}}}\big[R_k-R_{k-1}\big] = \mathrm E_{a_k}\mathrm E_{\widetilde{A_{k}}}\big[R_k-R_{k-1}\big] \label{eq7}
\end{align}
\begin{align*}
   \mathrm  E_{a_k}(f_k) - f_{k-1} =  \mathrm E_{a_k}(f_k - f_{k-1}).
\end{align*}
Hence, for full rank fixed $A_{k-1}$, it suffices to show that , 
\begin{align}
 \mathrm  E_{a_k}(f_k - f_{k-1}) \leq 0. \label{eq8} 
\end{align}
The remainder of this proof establishes $\eqref{eq8}$.\\
By Lemma \ref{Under_E(k)}, $f_k$ is given by
\begin{align*}
    f_k = \big\| G_n\big\|_F^2-\big\langle G_k^2,P_{G_k^TA_k^T} \big\rangle+\big\| G_k\big(A_kG_k\big)^{\dagger}\big\|_F^2\sum_{i = k+1}^n\sigma_i^2 = \big\| G_n\big\|_F^2-T_1(k)+T_2(k)\sum_{i = k+1}^n\sigma_i^2,
\end{align*}
where $T_1(k)=\big\langle G_k^2,P_{G_k^TA_k^T}\big\rangle$, $T_2(k)=\big\| G_k\big(A_kG_k\big)^{\dagger}\big\|_F^2$\\
Thus,
{\allowdisplaybreaks
\begin{align}
    f_k - f_{k-1}& = T_1(k-1)- T_1(k)-T_2(k-1)\sigma_k^2+\Big(T_2(k)-T_2(k-1)\Big)\sum_{i=k+1}^n\sigma_i^2 .\label{eq9}
\intertext{Define $M_{k} = A_kG_k = \Big[M_{k-1}, m_k\Big]$. By Lemma \ref{lemma2}, we know $$T_1(k)-T_1(k-1)=\sum_{i=1}^{k-1}\big(\sigma_i^2-\sigma_k^2\big)\big(P_{G_k^TA_k^T}-P_{G_{k-1}^TA_{k-1}^T}\big)_{ii}=\sum_{i=1}^{k-1}\big(\sigma_i^2-\sigma_k^2\big)\big(P_{M_k^T}-P_{M_{k-1}^T}\big)_{ii} \quad .$$}
\intertext{By Lemma \ref{lemma3}, we know
$$T_2(k)-T_2(k-1) = \sum_{i=1}^{k-1}\big(\sigma_i^2-\sigma_k^2\big)\Big(M_k^{\dagger}\big(M_k{^{\dagger}}\big)^T-M_{k-1}^{\dagger}\big(M_{k-1}{^{\dagger}}\big)^T\Big)_{ii}+\sigma_k^2\bigg(\Tr\Big(M_{k}^{\dagger}\big(M_{k}^{\dagger}\big)^T\Big)-\Tr \Big(M_{k-1}^{\dagger}\big(M_{k-1}^{\dagger}\big)^T\Big)\bigg).$$}
\intertext{By \eqref{eq23} we induce $T_2(k-1)$, 
plug  $T_1(k-1)-T_1(k), T_2(k)-T_2(k-1)$, and $T_2(k-1)$'s expressions into the equation \eqref{eq9}, we get}
    f_k - f_{k-1} &=- \sum_{i=1}^{k-1}\big(\sigma_i^2-\sigma_k^2\big)\Big(P_{M_k^T}-P_{M_{k-1}^T}\Big)_{ii}\nonumber \\
    & \quad - \sum_{i=1}^{k-1}\sigma_i^2\Big(M_{k-1}^{\dagger}\big(M_{k-1}{^{\dagger}}\big)^T\Big)_{ii}\sigma_k^2 \nonumber\\&\quad +\sum_{i=1}^{k-1}\big(\sigma_i^2-\sigma_k^2\big)\Big(M_{k}^{\dagger}\big(M_{k}{^{\dagger}}\big)^T-M_{k-1}^{\dagger}\big(M_{k-1}{^{\dagger}}\big)^T\Big)_{ii}\sum_{i=k+1}^n\sigma_i^2\nonumber\\&\quad+\sigma_k^2\bigg(\Tr \Big(M_{k}^{\dagger}\big(M_{k}{^{\dagger}}\big)^T\Big)-\Tr\Big(M_{k-1}^{\dagger}\big(M_{k-1}{^{\dagger}}\big)^T\Big)\bigg)\sum_{i=k+1}^n\sigma_i^2.
\end{align}}
By Lemma $\ref{lemma6}$ we know the upper bound for the last two terms is 0.
The first term  upper bound is: 
\begin{align*}
    -\sum_{i=1}^{k-1}\big(\sigma_i^2-\sigma_k^2\big)\Big(P_{M_k^T}-P_{M_{k-1}^T}\Big)_{ii}
    &= -\sum_{i=1}^{k-1}\sigma_i^2\Big(P_{M_k^T}-P_{M_{k-1}^T}\Big)_{ii}
    +\sum_{i=1}^{k-1}\sigma_k^2\Big(P_{M_k^T}-P_{M_{k-1}^T}\Big)_{ii}\\
    & \leq -\sum_{i=1}^{k-1}\sigma_i^2\Big(P_{M_k^T}-P_{M_{k-1}^T}\Big)_{ii},
\end{align*}
where the inequality follows from Lemma $\ref{lemma3}$.\\
Therefore, 
\begin{align*}
f_k -f_{k-1}  
& \leq 
    -\sum_{i=1}^{k-1}\sigma_i^2\Big(P_{M_k^T}-P_{M_{k-1}^T}\Big)_{ii}-\sum_{i=1}^{k-1}\sigma_i^2\Big(M_{k-1}^{\dagger}M_{k-1}{^{\dagger}}^T\Big)_{ii}\sigma_k^2
\end{align*}
From Lemma \ref{lemma7} we know that, 
\begin{align*}
    -\sum_{i=1}^{k-1}\sigma_i^2\Big(P_{M_k^T}-P_{M_{k-1}^T}\Big)_{ii} \leq \sum_{i=1}^{k-1}\sigma_i^2\bigg[M_{k-1}^T\bigg(\Big(M_{k-1}M_{k-1}^T\Big)^{-1} m_km_k^T \Big(M_{k-1}M_{k-1}^T\Big)^{-1}\bigg)M_{k-1}\bigg]_{ii} ,
\end{align*}
Observe that,
\begin{align*}
    \sum_{i=1}^{k-1}\sigma_i^2\Big(M_{k-1}^{\dagger}M_{k-1}{^{\dagger}}^T\Big)_{ii}\sigma_k^2 &= \sum_{i=1}^{k-1}\sigma_i^2\bigg(M_{k-1}^T\Big(M_{k-1}M_{k-1}^T\Big)^{-2}M_{k-1}\bigg)_{ii}\sigma_k^2\\
    & = \sum_{i=1}^{k-1}\sigma_i^2\bigg(M_{k-1}^T\Big(M_{k-1}M_{k-1}^T\Big)^{-1}I_m\sigma_k^2\Big(M_{k-1}M_{k-1}^T\Big)^{-1}M_{k-1}\bigg)_{ii}.
\end{align*}
Therefore, we have 
\begin{align}
   f_k -f_{k-1}& \leq 
    \sum_{i=1}^{k-1}\sigma_i^2\bigg(M_{k-1}^T\Big(M_{k-1}M_{k-1}^T\Big)^{-1}\Big(m_km_k^T- I_m\sigma_k^2\Big)\Big(M_{k-1}M_{k-1}^T\Big)^{-1}M_{k-1}\bigg)_{ii}. \label{eq10}
\end{align}

\noindent In order to establish \eqref{eq9}, we take the expectation of \eqref{eq10} with respect to $a_k$, and using $m_k = \sigma_ka_k$, and $a_k \sim N(0,I_m)$. Following the isotropy property $\mathrm E_{a_k}(m_km_k^T)= \sigma_k^2I_m$.
We obtain
{\allowdisplaybreaks
\begin{align}
\mathrm E_{a_k}(f_k-f_{k-1})   &\leq \mathrm E_{a_k}\Bigg[\sum_{i=1}^{k-1}\sigma_i^2\bigg(M_{k-1}^T\Big(M_{k-1}M_{k-1}^T\Big)^{-1}\Big(m_km_k^T- I_m\sigma_k^2\Big)\Big(M_{k-1}M_{k-1}^T\Big)^{-1}M_{k-1}\bigg)_{ii}\Bigg]\nonumber\\
    & = \sum_{i=1}^{k-1}\sigma_i^2\bigg(M_{k-1}^T\Big(M_{k-1}M_{k-1}^T\Big)^{-1}E_{a_k}\Big(m_km_k^T- I_m\sigma_k^2\Big)\Big(M_{k-1}M_{k-1}^T\Big)^{-1}M_{k-1}\bigg)_{ii}\label{eq12}\\
    & = 0 \nonumber.\qedhere
\end{align}}
\end{proof}
\subsection{Auxiliary Lemmas for the Underdetermined Regime}
The proof of monotonicity in the regime \(k\ge m\) relies on several
projection identities. We collect these auxiliary lemmas here.
\begin{lemma}\label{lemma2}
Let \( G_n = \operatorname{diag}(\sigma_1, \dots, \sigma_n) =\begin{pmatrix}
G_k & 0 \\
0 & \widetilde{G_k}
\end{pmatrix}\), where \( \sigma_i  \) is decreasing. Let $A \in \R ^{m\times n}$, $A_{ij}\stackrel{i.i.d}{\sim} \mathcal{N}(0,1)$. Then the following expectation identities hold:
\begin{align}
   \mathrm  E_A \left\| G_k P_{M_k^T} - G_k \right\|_F^2 + \| \widetilde{G_k }\|_F^2 =  \big\| G_n\big\|_F^2- \mathrm E_{A_k}\big\langle G_k^2,P_{M_k^T} \big\rangle.
\label{eq13}
\end{align}
\begin{align}
 \mathrm  E_{A} \left\| G_k (M_k)^\dagger M_k \right\|_F^2 =\mathrm E_{A_k} \left\| G_k (M_k)^\dagger \right\|_F^2 \sum_{i=k+1}^{n} \sigma_i^2.
\label{eq14}  
\end{align}
\end{lemma}
\begin{proof}[Proof of Lemma \ref{lemma2}.] 
First, we show \eqref{eq13}
\begin{align}
   \mathrm E_{A}\Big\| G_k P_{M_k^T} -G_k\Big\|_F^2 + \big\| \widetilde{G_k} \big\|_F^2 
 & =\mathrm E_{A}\bigg(\Big\| G_k P_{M_k^T} -G_k\Big\|_F^2 + \big\| \widetilde{G_k} \big\|_F^2\bigg) \nonumber \\
& =\mathrm  E_{A}\bigg(\Big \langle G_k P_{M_k^T} -G_k,G_k P_{M_k^T} -G_k\Big \rangle + \big\| \widetilde{G_k} \big\|_F^2\bigg)\nonumber \\
& =\mathrm E_{A}\bigg(\big\| G_k P_{M_k^T} \big\|_F^2 -2\Big\langle G_k P_{M_k^T},G_k \Big\rangle +\big\| G_k\big\|_F^2+ \big\| \widetilde{G_k} \big\|_F^2\bigg)\nonumber \\
 & = \mathrm E_{A}\bigg(\big\| G_k P_{M_k^T} \big\|_F^2 -2\Big\langle G_k P_{M_k^T}P_{M_k^T}^TG_k ^T,I\Big\rangle +\big\| G_k\big\|_F^2+ \big\| \widetilde{G_k} \big\|_F^2\bigg)\label{eq15}\\
 & = \mathrm E_{A}\bigg(\big\| G_k P_{M_k^T} \big\|_F^2 -2\Big\langle G_k P_{M_k^T},G_k P_{M_k^T}\Big\rangle +\big\| G_k\big\|_F^2+ \big\| \widetilde{G_k} \big\|_F^2\bigg) \nonumber 
 \\
 & =\mathrm E_{A}\bigg(\big\| G_k P_{M_k^T} \big\|_F^2 -2\big\| G_k P_{M_k^T} \big\|_F^2 +\big\| G_k\big\|_F^2+ \big\| \widetilde{G_k} \big\|_F^2\bigg)\nonumber  \\
 & =\mathrm  E_{A}\bigg(-\big\| G_k P_{M_k^T} \big\|_F^2 +\big\| G_k\big\|_F^2+ \big\| \widetilde{G_k} \big\|_F^2\bigg)\nonumber \\
 & =\mathrm  E_{A}\bigg(\big\| G_n\big\|_F^2- \Big\langle G_k^2,P_{M_k^T} \Big\rangle\bigg)\nonumber \\
 & = \big\| G_n\big\|_F^2- \mathrm E_{A_k}\Big\langle G_k^2,P_{M_k^T} \Big\rangle \nonumber
\end{align}
where $\eqref{eq15}$ follows by
$P^2=P$, and $P^T=P$. The last equality holds because $\big\| G_n\big\|_F^2- \Big\langle G_k^2,P_{M_k^T} \Big\rangle$ is independent of $\widetilde{A_k}$. \\
Next, we want to show \eqref{eq14}, 
\allowdisplaybreaks{\begin{align}
\mathrm E_{A}\Big\| G_k\big(M_k\big)^{\dagger}\widetilde{M_k}\Big\|_F^2&=
\mathrm E_{A_k}\mathrm E_{\widetilde{A_k}}\Big\| G_k\big(M_k\big)^{\dagger}\widetilde{M_k}\Big\|_F^2\nonumber\\
&=\mathrm E_{A_k}\mathrm E_{\widetilde{A_k}}\Big\langle G_k\big(M_k\big)^{\dagger}\widetilde{M_k},G_k\big(M_k\big)^{\dagger}\widetilde{M_k}\Big\rangle\nonumber\\
& =\mathrm E_{A_k} \bigg\langle \Big(G_k\big(M_k\big)^{\dagger}\Big)^TG_k\big(M_k\big)^{\dagger},\mathrm E_{\widetilde{A_k}}\Big(\widetilde{M_k}\widetilde{M_k}^T\Big)\bigg\rangle\nonumber\\
& =\mathrm E_{A_k} \bigg\langle \Big(G_k\big(M_k\big)^{\dagger}\Big)^TG_k\big(M_k\big)^{\dagger},\sum_{ i=k+1}^n\sigma_i^2\mathrm E_{a_i}\big(a_ia_i^T\big)\bigg\rangle \label{eq16}\\ 
& =\mathrm E_{A_k}\bigg\langle \Big(G_k\big(M_k\big)^{\dagger}\Big)^TG_k\big(M_k\big)^{\dagger},\sum_{i = k+1}^n\sigma_i^2I\bigg\rangle \label{eq17}\\
  & = \mathrm E_{A_k}\bigg\langle\Big(G_k\big(M_k\big)^{\dagger}\Big)^TG_k\big(M_k\big)^{\dagger},I\bigg\rangle\sum_{i = k+1}^n\sigma_i^2\nonumber\\
  & =\mathrm E_{A_k}\Big\| G_k\big(M_k\big)^{\dagger}\Big\|_F^2\sum_{i = k+1}^n\sigma_i^2\nonumber.
\end{align}}
Here, by the independence of columns of A, we can induce \eqref{eq16} to \eqref{eq17}.
\end{proof}
\begin{lemma}\label{lemma3}
Let $M_k =\begin{bmatrix}
    M_{k-1}, m_k
\end{bmatrix}$ 
 be a matrix $\in \R^{m \times k}$, where $M_{k} \in \R^{m \times k} $ has full rank, $m_k \in \R^m $, and $k > m$. $\forall i \in [k-1]$, we have $\big(P_{M_k^T}\big)_{ii} \leq \big(P_{M_{k-1}^T}\big)_{ii}$.
\end{lemma}
\begin{proof}[Proof of Lemma \ref{lemma3}.] Since it $M_k^T$ has full column rank, we can express this projector explicitly as 
\begin{align*}
    P_{M_k^T} &= M_k^T\big(M_kM_k^T\big)^{-1}M_k\\
    &=M_k^T\Bigg(\begin{bmatrix}
    M_{k-1} , m_k
\end{bmatrix}\begin{bmatrix}
    M_{k-1}^T\\m_k^T
\end{bmatrix}\Bigg)^{-1}M_k\\
&=M_k^T\Big(M_{k-1}M_{k-1}^T+m_km_k^T\Big)^{-1}M_k
\end{align*}
Evaluating the $(i,i)$ entry, where $i \in [k-1]$, let $e_i^{\prime} \in \R^{k-1},e_i \in \R^{k}$.
\begin{align*}
    \big(P_{M_k^T}\big)_{ii}&= e_i^TP_{M_k^T}e_i\\
    &=e_i^TM_k^T\Big(M_{k-1}M_{k-1}^T+m_km_k^T\Big)^{-1}M_ke_i\\
    &=e_i^T\begin{bmatrix}
        M_{k-1}^T \\ 0
    \end{bmatrix}\Big(M_{k-1}M_{k-1}^T+m_km_k^T\Big)^{-1}\begin{bmatrix}
        M_{k-1} & 0
    \end{bmatrix}e_i\\
    & \leq (e_i^{\prime})^TM_{k-1}^T\Big(M_{k-1}M_{k-1}^T\Big)^{-1}M_{k-1}e_i^{\prime}\\
    & = \big(P_{M_{k-1}^T}\big)_{ii}  
\end{align*}
The inequality follows from  $M_{k-1}M_{k-1}^T+m_km_k^T \succeq M_{k-1}M_{k-1}^T$, so $\Big(M_{k-1}M_{k-1}^T+m_km_k^T\Big)^{-1} \preceq  \Big(M_{k-1}M_{k-1}^T\Big)^{-1}$.
\end{proof}
We now use the above lemma to prove the following lemma.
\begin{lemma}\label{lemma4}
    Under the same assumptions as Lemma \ref{Under_E(k)}, for fixed $m \in [1,n]$, $k \in [m+1,n]$, let $T_1(k)=\big\langle G_k^2,P_{G_k^TA_k^T} \big\rangle=\big\langle G_k^2,P_{M_k^T} \big\rangle$  ,
\begin{equation}
T_1(k)-T_1(k-1) = \sum_{i=1}^{k-1}\big(\sigma_i^2-\sigma_k^2\big)\big(P_{M_k^T}-P_{M_{k-1}^T}\big)_{ii} \leq 0 \label{eq18}
\end{equation}
and $T_1(k)$ is decreasing monotonically.
\end{lemma}
\allowdisplaybreaks{
\begin{proof}[Proof of Lemma \ref{lemma4}.] 
Observe that, 
\begin{align*}
    T_1(k)&=\big\langle G_k^2,P_{M_k^T}\big\rangle\\& =\sum_{i=1}^{k-1}\sigma_i^2\big(P_{M_k^T}\big)_{ii}+\sigma_k^2\big(P_{M_k^T}\big)_{kk}\\&=\sum_{i=1}^{k-1}\sigma_i^2\big[P_{M_{k-1}^T}+P_{M_k^T}-P_{M_{k-1}^T}\big]_{ii}+\sigma_k^2\big(P_{M_{k}^T}\big)_{kk}\\&=T_1(k-1)+\sum_{i=1}^{k-1}\sigma_i^2\big(P_{M_k^T}-P_{M_{k-1}^T}\big)_{ii}+\sigma_k^2\big(P_{M_k^T}\big)_{kk}
\end{align*}
Noting that \( \Tr(P_{M_k^T}) = \sum_{i=1}^{k}(P_{M_k^T})_{ii}=m \),
\begin{align*}
    T_1(k)-T_1(k-1)&=\sum_{i=1}^{k-1}\sigma_i^2\big(P_{M_k^T}-P_{M_{k-1}^T}\big)_{ii}+\sigma_k^2\big(P_{M_k^T}\big)_{kk}\\
    &= \sum_{i=1}^{k-1}\sigma_i^2\big(P_{M_k^T}-P_{M_{k-1}^T}\big)_{ii}+\sigma_k^2\Big(m-\sum_{i=1}^{k-1}\big(P_{M_k^T}\big)_{ii}\Big)\\
    & = \sum_{i=1}^{k-1}\sigma_i^2\big(P_{M_k^T}-P_{M_{k-1}^T}\big)_{ii}+\sigma_k^2\Big(\sum_{i=1}^{k-1}\big(P_{M_{k-1}^T}\big)_{ii}-\sum_{i=1}^{k-1}\big(P_{M_k^T}\big)_{ii}\Big)\\
    &=\sum_{i=1}^{k-1}\sigma_i^2\big(P_{M_k^T}-P_{M_{k-1}^T}\big)_{ii}-\sigma_k^2\sum_{i=1}^{k-1}\big(P_{M_k^T}-P_{M_{k-1}^T}\big)_{ii}\\
    &=\sum_{i=1}^{k-1}\big(P_{M_k^T}-P_{M_{k-1}^T}\big)_{ii}\big(\sigma_i^2-\sigma_k^2\big)\\
    &\leq 0
\end{align*}
The second equality follows \( \Tr(P_{M_k^T}) = \sum_{i=1}^{k}(P_{M_k^T})_{ii}=\sum_{i=1}^{k-1}(P_{M_k^T})_{ii}+ (P_{M_k^T})_{kk} =m \). \\
The third equality follows $\Tr(P_{M_{k-1}^T})=\sum_{i=1}^{k-1}\big(P_{M_{k-1}^T}\big)_{ii} =m$.
\end{proof}}
\begin{lemma}\label{lemma5}
     Let $M_k = \begin{bmatrix}
        M_{k-1},m_k
    \end{bmatrix} \in \R^{m\times k}$ 
 be a matrix $\in \R^{m \times k}$, where $M_{k-1} \in \R^{m \times (k-1)}$ has full rank, $m_k \in \R^m $ and $k \geq m$.  
 \begin{equation}
     \Tr\Big(M_{k-1}^{\dagger}(M_{k-1}{^{\dagger}})^T\Big) \geq \Tr\Big(M_k^{\dagger}(M_k{^{\dagger}})^T\Big)\label{eq19} 
 \end{equation}
and 
\begin{equation}
 \Big(M_{k-1}^{\dagger}(M_{k-1}{^{\dagger}})^T\Big)_{ii} \geq \Big(M_k^{\dagger}(M_k{^{\dagger}})^T \Big)_{ii}  \label{eq20} 
\end{equation}
$\forall i \in [k-1]$. 
\end{lemma}
\allowdisplaybreaks{
\begin{proof}[Proof of Lemma \ref{lemma5}.]
\begin{align*}
        M_k^{\dagger} &= M_k^T\Big(M_kM_k^T\Big)^{-1}\\
        &= \begin{pmatrix}
            M_{k-1}^T\\m_k^T
        \end{pmatrix}\Big(M_{k-1}M_{k-1}^T+m_km_k^T\Big)^{-1},\\
        \intertext{here $M_{k-1}M_{k-1}^T$, $m_km_k^T$ all PSD matrix.}
    M_k^{\dagger}(M_k{^{\dagger}})^T&=\begin{pmatrix}
            M_{k-1}^T\\m_k^T
        \end{pmatrix}\Big(M_{k-1}M_{k-1}^T+m_km_k^T\Big)^{-1}\Big(M_{k-1}M_{k-1}^T+m_km_k^T\Big)^{-1}\begin{pmatrix}
            M_{k-1} & m_k
        \end{pmatrix}\\
    & = \begin{pmatrix}
            M_{k-1}^T\\m_k^T
        \end{pmatrix}\Big(M_{k-1}M_{k-1}^T+m_km_k^T\Big)^{-2}\begin{pmatrix}
            M_{k-1} & m_k
        \end{pmatrix}\\
    & \preccurlyeq \begin{pmatrix}
            M_{k-1}^T\\m_k^T
        \end{pmatrix}\Big(M_{k-1}M_{k-1}^T\Big)^{-2}\begin{pmatrix}
            M_{k-1} & m_k
        \end{pmatrix}\\
    & = \begin{pmatrix}
            M_{k-1}^T(M_{k-1}M_{k-1}^T)^{-2}M_{k-1} & M_{k-1}^T(M_{k-1}M_{k-1}^T)^{-2}m_k\\m_k^T(M_{k-1}M_{k-1}^T)^{-2}M_{k-1} &m_k^T(M_{k-1}M_{k-1}^T)^{-2}m_k
        \end{pmatrix}
\end{align*}
We are using trace to investigate the behavior of $M_k^{\dagger}(M_k{^{\dagger}})^T$ and $ M_{k-1}^{\dagger}(M_{k-1}{^{\dagger}})^T$ on both sides.
\begin{align*}
   \Tr\bigg(M_{k-1}^{\dagger}\Big(M_{k-1}{^{\dagger}}\Big)^T\bigg) & = \Tr\bigg(M_{k-1}^T\Big(M_{k-1}M_{k-1}^T\Big)^{-2}M_{k-1}\bigg)\\
    &= \bigg\langle I, M_{k-1}^T\Big(M_{k-1}M_{k-1}^T\Big)^{-1}\Big(M_{k-1}M_{k-1}^T\Big)^{-1}M_{k-1}\bigg\rangle\\
    &=\bigg\langle M_{k-1}^TM_{k-1}, \Big(M_{k-1}M_{k-1}^T\Big)^{-2}\bigg\rangle\\
    & =\bigg \langle I, \Big(M_{k-1}M_{k-1}^T\Big)^{-1}\bigg\rangle\\
\Tr\bigg(M_k^{\dagger}\big(M_k{^{\dagger}}\big)^T\bigg) & = \Tr\bigg(M_{k-1}^T\Big(M_{k-1}M_{k-1}^T+m_km_k^T\Big)^{-2}M_{k-1}+m_k^T\Big(M_{k-1}M_{k-1}^T+m_km_k^T\Big)^{-2}m_k\bigg)
    \\
    &=\bigg\langle I, M_{k-1}^T\Big(M_{k-1}M_{k-1}^T+m_km_k^T\Big)^{-2}M_{k-1}\bigg\rangle + \bigg\langle I, m_k^T\Big(M_{k-1}M_{k-1}^T+m_km_k^T\Big)^{-2}m_k\bigg\rangle\\
    &=\bigg\langle M_{k-1}^TM_{k-1}, \Big(M_{k-1}M_{k-1}^T+m_km_k^T\Big)^{-2}\bigg\rangle +\bigg\langle m_k^Tm_k, \Big(M_{k-1}M_{k-1}^T+m_km_k^T\Big)^{-2}\bigg\rangle\\
    &=\bigg\langle M_{k-1}^TM_{k-1}+m_k^Tm_k, \Big(M_{k-1}M_{k-1}^T+m_km_k^T\Big)^{-2}\bigg\rangle \\
    &=\bigg\langle I ,\Big(M_{k-1}M_{k-1}^T+m_km_k^T\Big)^{-1}\bigg\rangle 
\end{align*}
Since $\Big(M_{k-1}M_{k-1}^T+m_km_k^T\Big)^{-1}\preccurlyeq \Big(M_{k-1}M_{k-1}^T\Big)^{-1}$, 
we proved $\Tr\Big(M_{k-1}^{\dagger}\big(M_{k-1}{^{\dagger}}\big)^T\Big) \geq \Tr\Big(M_k^{\dagger}\big(M_k{^{\dagger}}\big)^T\Big)$.\\
For entry-wise perspective, we have:
\begin{align*}
    \Big(M_k^{\dagger}\big(M_k{^{\dagger}}\big)^T \Big)_{ii} &= \bigg\langle e_ie_i^T, M_k^T\Big(M_{k-1}M_{k-1}^T+m_km_k^T\Big)^{-2}M_k\bigg\rangle\\
    & =  \bigg\langle M_ke_ie_i^TM_k^T, \Big(M_{k-1}M_{k-1}^T+m_km_k^T\Big)^{-2}\bigg\rangle\\
    \intertext{Since $i \in [k-1]$, So we can use $M_{k-1}$ substitute $M_k$. }
    &=\bigg\langle M_{k-1}e_ie_i^TM_{k-1}^T, \Big(M_{k-1}M_{k-1}^T+m_km_k^T\Big)^{-2}\bigg\rangle\\
    & = \bigg\langle e_ie_i^T, M_{k-1}^T\Big(M_{k-1}M_{k-1}^T+m_km_k^T\Big)^{-2}M_{k-1}\bigg\rangle\\
    & \leq \bigg\langle e_ie_i^T, M_{k-1}^T\Big(M_{k-1}M_{k-1}^T\Big)^{-2}M_{k-1}\bigg\rangle\\
    & = \bigg(M_{k-1}^{\dagger}\Big(M_{k-1}{^{\dagger}}\Big)^T \bigg)_{ii}
\end{align*}
which proves the following Lemma.
\end{proof}}
\begin{lemma}\label{lemma6} Under Lemma \ref{Under_E(k)} assumptions, let $T_2(k)=\Big\|G_k(A_kG_k)^{\dagger}\Big\|_F^2$. For all $ m \leq k$,
\begin{align}
  T_2(k) - T_2(k-1)&= \sum_{i=1}^{k-1}(\sigma_i^2-\sigma_k^2)\Big(M_k^{\dagger}\big(M_k{^{\dagger}}\big)^T-M_{k-1}^{\dagger}\big(M_{k-1}{^{\dagger}}\big)^T\Big)_{ii}+\sigma_k^2\bigg(\Tr\Big(M_{k}^{\dagger}\big(M_{k}{^{\dagger}}\big)^T\Big)-\Tr \Big(M_{k-1}^{\dagger}\big(M_{k-1}{^{\dagger}}\big)^T\Big)\bigg)\nonumber\\
  &\leq 0 \label{eq21}
\end{align}
and as a consequence,
\begin{align}
\Big\|G_k(A_kG_k)^{\dagger}\Big\|_F^2\sum_{i=k+1}^n\sigma_i^2   \label{eq22}
\end{align}
is monotonically decreasing.
\end{lemma}
\begin{proof}[Proof of Lemma \ref{lemma6}.]
    Observe that $\sum_{i=k+1}^n\sigma_i^2$ decreases as $k$ increases. Therefore, to establish the desired result, it is sufficient to show that $\Big\|G_k(A_kG_k)^{\dagger}\Big\|_F^2$ is a decreasing function of k.
Thus, we want to show:
\begin{align*}
    \bigg\langle G_{k-1}^2, M_{k-1}^{\dagger}\Big(M_{k-1}{^{\dagger}}^T\Big)\bigg\rangle \geq \bigg\langle G_{k-1}^2, M_{k-1}^T\Big(M_{k-1}M_{k-1}^T+m_km_k^T\Big)^{-2}M_{k-1}\bigg\rangle + \sigma_k^2\bigg(m_k^T\Big(M_{k-1}M_{k-1}^T+m_km_k^T\Big)^{-2}m_k\bigg)
\end{align*}
Recall $M_k^{\dagger} =\big(A_kG_k\big)^{\dagger} \in \R^{k \times m} $  \\
\begin{align*}
  T_2(k)=\Big\|G_k\big(A_kG_k\big)^{\dagger}\Big\|_F^2=\Big\langle G_k^2,\big(A_kG_k\big)^{\dagger}\big(A_kG_k\big){^{\dagger}}^T\Big\rangle =\Big\langle G_k^2,M_k^{\dagger}\big(M_k{^{\dagger}}\big)^T\Big\rangle   
\end{align*}
{\allowdisplaybreaks
\begin{align}
    T_2(k)&=\Big\langle G_k^2,M_k^{\dagger}\big(M_k{^{\dagger}}\big)^T\Big\rangle \nonumber \\
    & = \sum_{i=1}^k \sigma_i^2\Big(M_k^{\dagger}\big(M_k{^{\dagger}}\big)^T\Big)_{ii} \label{eq23}\\
    & = \sum_{i=1}^{k-1} \sigma_i^2\Big(M_k^{\dagger}\big(M_k{^{\dagger}}\big)^T\Big)_{ii}+\sigma_k^2\Big(M_k^{\dagger}\big(M_k{^{\dagger}}\big)^T\Big)_{kk} \nonumber\\
    & = \sum_{i=1}^{k-1} \sigma_i^2\Big(M_{k-1}^{\dagger}\big(M_{k-1}{^{\dagger}}\big)^T+M_k^{\dagger}\big(M_k{^{\dagger}}\big)^T-M_{k-1}^{\dagger}\big(M_{k-1}{^{\dagger}}\big)^T\Big)_{ii}
+\sigma_k^2\Big(M_k^{\dagger}\big(M_k{^{\dagger}}\big)^T\Big)_{kk} \nonumber\\
    & = T_2(k-1) + \sum_{i=1}^{k-1} \sigma_i^2\Big(M_k^{\dagger}\big(M_k{^{\dagger}}\big)^T-M_{k-1}^{\dagger}\big(M_{k-1}{^{\dagger}}\big)^T\Big)_{ii}
+\sigma_k^2\Big(M_k^{\dagger}\big(M_k{^{\dagger}}\big)^T\Big)_{kk}\nonumber\\
    T_2(k) -T_2(k-1) & = \sum_{i=1}^{k-1} \sigma_i^2\Big(M_k^{\dagger}\big(M_k{^{\dagger}}\big)^T-M_{k-1}^{\dagger}\big(M_{k-1}{^{\dagger}}\big)^T\Big)_{ii}
+\sigma_k^2\Big(M_k^{\dagger}\big(M_k{^{\dagger}}\big)^T\Big)_{kk}\nonumber\\
    & =\sum_{i=1}^{k-1} \sigma_i^2\Big(M_k^{\dagger}\big(M_k{^{\dagger}}\big)^T-M_{k-1}^{\dagger}\big(M_{k-1}{^{\dagger}}\big)^T\Big)_{ii}
\nonumber\\&\quad+\sigma_k^2\bigg(\Big(M_k^{\dagger}\big(M_k{^{\dagger}}\big)^T\Big)_{kk}+\sum_{i=1}^{k-1}\Big( M_{k-1}^{\dagger}\big(M_{k-1}{^{\dagger}}\big)^T-M_k^{\dagger}\big(M_k{^{\dagger}}\big)^T\Big)_{ii}\nonumber\\&\quad \quad -\sum_{i=1}^{k-1}\Big( M_{k-1}^{\dagger}\big(M_{k-1}{^{\dagger}}\big)^T-M_k^{\dagger}\big(M_k{^{\dagger}}\big)^T \Big)_{ii}\bigg)\nonumber\\
    & = \sum_{i=1}^{k-1}\big(\sigma_i^2-\sigma_k^2\big)\Big(M_k^{\dagger}\big(M_k{^{\dagger}}\big)^T-M_{k-1}^{\dagger}\big(M_{k-1}{^{\dagger}}\big)^T\Big)_{ii}\nonumber\\&\quad+\sigma_k^2\bigg(\Big(M_k^{\dagger}(M_k{^{\dagger}})^T\Big)_{kk}-\sum_{i=1}^{k-1} \Big(M_{k-1}^{\dagger}(M_{k-1}{^{\dagger}})^T-M_k^{\dagger}(M_k{^{\dagger}})^T\Big)_{ii}\bigg)\nonumber\\
    & = \sum_{i=1}^{k-1}\big(\sigma_i^2-\sigma_k^2\big)\Big(M_k^{\dagger}\big(M_k{^{\dagger}}\big)^T-M_{k-1}^{\dagger}\big(M_{k-1}{^{\dagger}}\big)^T\Big)_{ii}\nonumber\\&\quad+\sigma_k^2\bigg(\Tr\Big(M_{k}^{\dagger}\big(M_{k}{^{\dagger}}\big)^T\Big)-\Tr \Big(M_{k-1}^{\dagger}\big(M_{k-1}{^{\dagger}}\big)^T\Big)\bigg)\nonumber\\
\intertext{
Based on Lemma \ref{lemma5}, we have
\begin{equation*}
    \sum_{i=1}^{k-1} \Big(M_k^{\dagger}(M_k{^{\dagger}})^T\Big)_{ii}  \leq \sum_{i=1}^{k-1} \Big(M_{k-1}^{\dagger}(M_{k-1}{^{\dagger}})^T\Big)_{ii} 
\end{equation*}
\begin{equation*}
    \Tr\Big(M_{k}^{\dagger}(M_{k}{^{\dagger}})^T\Big)\leq \Tr \Big(M_{k-1}^{\dagger}(M_{k-1}{^{\dagger}})^T\Big)
\end{equation*}
}
T_2(k) -T_2(k-1) &\leq 0.\qedhere
\end{align}}
\end{proof}
\begin{lemma}\label{lemma7}
    Under Lemma \ref{Under_E(k)} assumption, let  \(M_k = A_kG_k \) be a sequence of matrices evolving as \( M_k = [M_{k-1}, m_k] \), where \( m_k = \sigma_ka_k \) is an additional column vector. The projection matrices associated with these matrices are given by:
\[
P_{G_k^T A_k^T} =P_{M_k^T}= M_k^T (M_k M_k^T)^{-1} M_k, \quad
P_{G_{k-1}^T A_{k-1}^T}=P_{M_{k-1}^T} = M_{k-1}^T (M_{k-1} M_{k-1}^T)^{-1} M_{k-1}.
\]
Then,
\begin{align}
 \sum_{i=1}^{k-1}\sigma_i^2\Big(P_{M_k^T}-P_{M_{k-1}^T}\Big)_{ii} 
  \geq -\sum_{i=1}^{k-1}\sigma_i^2\bigg(M_{k-1}^T\bigg(\Big(M_{k-1}M_{k-1}^T\Big)^{-1} m_km_k^T \Big(M_{k-1}M_{k-1}^T\Big)^{-1}\bigg)M_{k-1}\bigg)_{ii}\label{eq24}.
\end{align}
\end{lemma}
\begin{proof}[Proof of Lemma \ref{lemma7}.] 
Since 
{\allowdisplaybreaks
\begin{align*}
  \sum_{i=1}^{k-1}\sigma_i^2\Big(P_{G_k^TA_k^T}-P_{G_{k-1}^TA_{k-1}^T}\Big)_{ii}
  & = \sum_{i=1}^{k-1}\sigma_i^2\bigg(\Big(M_k^T\Big(M_kM_k^T\Big)^{-1}M_k\Big)_{ii}-\Big(M_{k-1}^T\Big(M_{k-1}M_{k-1}^T\Big)^{-1}M_{k-1}\Big)_{ii}\bigg) \\&=\sum_{i=1}^{k-1}\sigma_i^2\bigg(\Big(M_{k-1}^T\Big(M_kM_k^T\Big)^{-1}M_{k-1}\Big)_{ii}-\Big(M_{k-1}^T\Big(M_{k-1}M_{k-1}^T\Big)^{-1}M_{k-1}\Big)_{ii}\bigg)\\
  & = \sum_{i=1}^{k-1}\sigma_i^2\bigg(M_{k-1}^T\bigg(\Big(\Big(M_kM_k^T\Big)^{-1}-\Big(M_{k-1}M_{k-1}^T\Big)^{-1}\bigg)M_{k-1}\bigg)_{ii}
\end{align*}}
We know that $\Big(M_kM_k^T\Big)^{-1} = \Big(M_{k-1}M_{k-1}^T+m_km_k^T\Big)^{-1}$,\\
By Lemma \ref{lemma13}, we know that:
\begin{align*}
    \Big(M_{k-1}M_{k-1}^T+m_km_k^T\Big)^{-1}-\Big(M_{k-1}M_{k-1}^T\Big)^{-1}+\Big(M_{k-1}M_{k-1}^T\Big)^{-1} m_km_k^T \Big(M_{k-1}M_{k-1}^T\Big)^{-1} \succcurlyeq 0
\end{align*}
Therefore,
\begin{align*}
\sum_{i=1}^{k-1}\sigma_i^2\Big(P_{G_k^TA_k^T}-P_{G_{k-1}^TA_{k-1}^T}\Big)_{ii}=&\quad \sum_{i=1}^{k-1}\sigma_i^2\bigg(M_{k-1}^T\bigg(\Big(M_kM_k^T\Big)^{-1}-\Big(M_{k-1}M_{k-1}^T\Big)^{-1}\bigg)M_{k-1}\bigg)_{ii}\\
     &= \sum_{i=1}^{k-1}\sigma_i^2\bigg(M_{k-1}^T\bigg(\Big(M_{k-1}M_{k-1}^T+m_km_k^T\Big)^{-1}-\Big(M_{k-1}M_{k-1}^T\Big)^{-1}\bigg)M_{k-1}\bigg)_{ii}\\
     & \geq -\sum_{i=1}^{k-1}\sigma_i^2\bigg(M_{k-1}^T \Big(M_{k-1}M_{k-1}^T\Big)^{-1} m_km_k^T \Big(M_{k-1}M_{k-1}^T\Big)^{-1}M_{k-1}\bigg)_{ii}.
\end{align*}
\end{proof}
\subsection{Main Result for Overdetermined System}\label{secov}
We now consider the overdetermined regime \(1\le k\le m\).
Our goal in this section is to derive an explicit expression for
\(E(k)\) and use it to compare \(E(k)\) with the full-model error \(E(n)\).

For $k \leq m$, we know that
$\hat{z}(k) =\underset{z_k\in \R^k}{\argmin} \ \frac{1}{2} \left\| A\mathcal{G}_k z_k - y \right\|_2^2 $.
Since \(A_kG_k\) has full column rank with probability $1$, solving this
least-squares problem gives $\hat{z}(k) = \big(A_kG_k\big)^{\dagger}AG_nz^\star$, where $\hat{z}(k)\in \R^k$, $\big(A_kG_k\big)^{\dagger}=\Big(G_k^TA_k^TA_kG_k\Big)^{-1}G_k^TA_k^T$. Here $A_k \in \R^{m \times k}, G_k \in \R^{k \times k}$. 
\begin{lemma}\label{lemma8}
    Suppose $x^\star \thicksim \mathcal{G}_n$, $G_n = diag(\sigma_1,\dots,\sigma_n) \in \R ^{n\times n}$, such that $\sigma_1 \geq \sigma_2 \geq \dots \geq \sigma_n >0 $. For any $k \leq m \leq n$, the random matrix $A \in \R^{m\times n}$ with i.i.d. $\mathcal{N}(0,1)$ entries, 
    the MAP estimator obeys
\begin{align}
    \mathrm E(k)&=\mathrm E_A\mathrm E_{x^\star\thicksim \mathcal{G}_n}\big\Arrowvert \hat{x}(k) -x^\star \big\Arrowvert_{2}^{2}\nonumber\\
    &=\sum_{i=k+1}^n\sigma_i^2+\mathrm E_{A_k}\Big\Arrowvert G_k\big(A_kG_k\big)^{\dagger}\Big\Arrowvert_F^2\sum_{i=k+1}^n\sigma_i^2 \label{eq:E(k)over}
\end{align} 

\end{lemma}
\begin{proof}[Proof of Lemma \ref{lemma8}.] 
{\allowdisplaybreaks
\begin{align*}
    \mathrm E(k) &:= \mathrm E_A\mathrm E_{x^\star\thicksim \mathcal{G}_n}\big\Arrowvert \hat{x}(k) -x^\star \big\Arrowvert_{2}^{2}\\
&=\mathrm E_A\mathrm E_{z^\star\thicksim \mathcal{N}(0,I_n)}\Big\Arrowvert \mathcal{G}_k\hat{z}(k)-G_nz^\star \Big\Arrowvert_{2}^{2}\\
&=\mathrm E_A\mathrm E_{z^\star}\bigg\Arrowvert \binom{\mathrm{G}_k}{0}\hat{z}(k)-\begin{pmatrix}
G_k & 0 \\
0 & \widetilde{G_k}
\end{pmatrix}z^\star\bigg\Arrowvert_{2}^{2} \\
&=\mathrm E_A\mathrm E_{z^\star}\Big\Arrowvert G_k\hat{z}(k)-G_k z^\star_{k}\Big\Arrowvert_2^2+ \sum_{i=k+1}^n \sigma_i^2\\
&=\mathrm E_A\mathrm E_{z^\star}\Big\Arrowvert G_k\big(A_kG_k\big)^{\dagger}AG_nz^\star-G_kz_k^\star \Big\Arrowvert_2^2+ \sum_{i=k+1}^n \sigma_i^2\\
&=\mathrm  E_A\mathrm E_{z^\star}\Big\Arrowvert G_k\big(A_kG_k\big)^{\dagger}\Big(A_kG_k z^\star_{k}+\widetilde{A_k}\widetilde{G_k}\widetilde{z^\star_{k}}\Big)-G_k z^\star_{k}\Big\Arrowvert_2^2+ \sum_{i=k+1}^n \sigma_i^2\\
& = \mathrm E_A\mathrm E_{z^\star}\Big\Arrowvert G_k\Big(G_k^TA_k^TA_kG_k\Big)^{-1}G_k^TA_k^T\Big(A_kG_k z^\star_{k}+\widetilde{A_k}\widetilde{G_k}\widetilde{z^\star_{k}}\Big)-G_k z^\star_{k}\Big\Arrowvert_2^2+ \sum_{i=k+1}^n \sigma_i^2\\
& = \mathrm E_A\mathrm E_{z^\star}\Big\Arrowvert G_k\Big(G_k^TA_k^TA_kG_k\Big)^{-1}G_k^TA_k^TA_kG_k z^\star_{k}+G_k\Big(G_k^TA_k^TA_kG_k\Big)^{-1}G_k^TA_k^T\widetilde{A_k}\widetilde{G_k}\widetilde{z^\star_{k}}-G_k z^\star_{k}\Big\Arrowvert_2^2\\&\qquad+ \sum_{i=k+1}^n \sigma_i^2\\
&=\mathrm E_A\mathrm E_{z^\star}\Big\Arrowvert G_k\big(A_kG_k\big)^{\dagger}\widetilde{A_k}\widetilde{G_k}\widetilde{z^\star_{k}}\Big\Arrowvert_2^2
+\sum_{i=k+1}^n\sigma_i^2\\
& = \sum_{i=k+1}^n \sigma_i^2 + \mathrm E_A\Big\Arrowvert G_k\big(A_kG_k\big)^{\dagger}\widetilde{A_k}\widetilde{G_k}\Big\Arrowvert_F^2
\\
& =\sum_{i=k+1}^n\sigma_i^2 + \mathrm E_{A_k}\Big\Arrowvert G_k\big(A_kG_k\big)^{\dagger}\Big\Arrowvert_F^2\sum_{i=k+1}^n\sigma_i^2. \qedhere
\end{align*}}   
\end{proof}
The proof of Lemma~\ref{lemma:closed-form-k-less-m} relies on the following
result of~\cite{halko2011finding}. We use the Frobenius-norm specialization
of \cite[Theorem~10.5]{halko2011finding} and include its proof below for the
reader's convenience. The argument follows the proof of
\cite[Theorem~10.5]{halko2011finding} and refers to
\cite[Theorem~9.1 and Propositions~10.1--10.2]{halko2011finding}.
\begin{lemma}[Special case of the result in \cite{halko2011finding}]
\label{lem:hmt-frob-specialized}
Let
$\Sigma=\operatorname{diag}(\sigma_1,\ldots,\sigma_n),    \sigma_1\ge\cdots\ge\sigma_n>0,$
and fix \(1\le k\le m-2\). Partition
\[
    \Sigma=
    \begin{pmatrix}
    \Sigma_1 & 0\\
    0 & \Sigma_2
    \end{pmatrix},
    \qquad
    \Sigma_1\in\mathbb R^{k\times k},
    \qquad
    \Sigma_2\in\mathbb R^{(n-k)\times(n-k)}.
\]
Let \(\Omega\in\mathbb R^{n\times m}\) be a standard Gaussian matrix and
partition it as
\[
    \Omega=
    \begin{pmatrix}
    \Omega_1\\
    \Omega_2
    \end{pmatrix},
    \qquad
    \Omega_1\in\mathbb R^{k\times m},
    \qquad
    \Omega_2\in\mathbb R^{(n-k)\times m}.
\]
Let \(Y=\Sigma\Omega\), and let \(P_Y\) denote the orthogonal projector
onto \(\operatorname{range}(Y)\). Then
\[
   \mathrm E_{\Omega}
    \bigl\|(I-P_Y)\Sigma\bigr\|_F^2
    \le
    \left(1+\frac{k}{m-k-1}\right)
    \|\Sigma_2\|_F^2 .
\]
\end{lemma}

\begin{proof}[Proof of Lemma \ref{lem:hmt-frob-specialized}.]
We follow the proof of
\cite[Theorem~10.5]{halko2011finding}, specialized to the present
Frobenius-norm setting and notation.

Since \(m\ge k+2\), the matrix \(\Omega_1\) has full row rank almost
surely. Applying the deterministic range-finder bound
\cite[Theorem~9.1]{halko2011finding} with the Frobenius norm gives
\[
    \bigl\|(I-P_Y)\Sigma\bigr\|_F^2
    \le
    \|\Sigma_2\|_F^2
    +
    \|\Sigma_2\Omega_2\Omega_1^\dagger\|_F^2 .
\]
Conditioning on \(\Omega_1\), the matrix \(\Omega_2\) is an independent
standard Gaussian matrix. By the Gaussian Frobenius moment identity
\cite[Proposition~10.1]{halko2011finding},
\[
    \mathrm E_{\Omega_2}
    \|\Sigma_2\Omega_2\Omega_1^\dagger\|_F^2
    =
    \|\Sigma_2\|_F^2
    \|\Omega_1^\dagger\|_F^2 .
\]
It remains to average over \(\Omega_1\). Since \(\Omega_1\) has full row
rank,
$\|\Omega_1^\dagger\|_F^2 =
\operatorname{tr}\bigl((\Omega_1\Omega_1^\top)^{-1}\bigr).$
Moreover, with Lemma~\ref{lem:iw-moment},
$\Omega_1\Omega_1^\top\sim \operatorname{Wishart}_k(I_k,m).$
Using the mean of inverse-Wishart,
\[
    \mathrm E(\Omega_1\Omega_1^\top)^{-1}
    =
    \frac{1}{m-k-1}I_k,
\]
which is finite because \(m>k+1\), we obtain
\[
    \mathrm E_{\Omega_1}\|\Omega_1^\dagger\|_F^2
    =
    \frac{k}{m-k-1}.
\]
Therefore
\[
    \mathrm E_{\Omega}
    \|\Sigma_2\Omega_2\Omega_1^\dagger\|_F^2
    =
    \frac{k}{m-k-1}\|\Sigma_2\|_F^2 .
\]
Taking expectation in the deterministic range-finder bound yields
\[
    \mathrm E_{\Omega}
    \bigl\|(I-P_Y)\Sigma\bigr\|_F^2
    \le
    \left(1+\frac{k}{m-k-1}\right)
    \|\Sigma_2\|_F^2 .
\]
\end{proof}
We now prove Lemma~\ref{lemma:closed-form-k-less-m}.
After deriving the exact expression for \(E(k)\) in the overdetermined regime,
we can verify $E(n) \leq E(k)$ for $1\le k <m-1$.
\allowdisplaybreaks{
\begin{proof}[Proof of Lemma~\ref{lemma:closed-form-k-less-m}.] 
Fix \(1\leq k\leq m-2\). 
By Lemma~\ref{lemma8}, we have
\[\mathrm E(k)=\sum_{i=k+1}^n\sigma_i^2+\mathrm E_{A_k}\Big\Arrowvert G_k\big(A_kG_k\big)^{\dagger}\Big\Arrowvert_F^2\sum_{i=k+1}^n\sigma_i^2.\]
Now, by Lemma~\ref{lemma9}, we get  
\begin{align}
    \mathrm E(k)= \sum_{i=k+1}^n\sigma_i^2+ \frac{k}{m-k-1} \sum_{i=k+1}^n\sigma_i^2\label{eq:exact-Ek-overdetermined}
\end{align}
It remains to compare \(E(k)\) with the full-model error \(E(n)\). By the
definition of the full-model estimator,
\begin{align*}
    \mathrm E(n) & = \mathrm E_{A,x^\star}\|\hat{x}(n)-x^\star\|_2^2 \\
    & = \mathrm E_{A,z^\star}\|G_nG_n^TA^T(AG_nG_n^TA^T)^{-1}AG_nz^\star-G_nz^\star\|_2^2  \\
    & = \mathrm E_{A}\|G_nG_n^TA^T(AG_nG_n^TA^T)^{-1}AG_n-G_n\|_F^2  \\
    & =\mathrm E_{A}\Big\|G_n\Big(\mathrm{P}_{G^T_nA^T}-I_n\Big)\Big\|_F^2 \\
    & = \mathrm E_{A}\Big\|\Big(I_n-\mathrm{P}_{G^T_nA^T}\Big)G_n\Big\|_F^2,
\end{align*}
where the last equality follows from the invariance of the Frobenius norm
under transposition, together with the symmetry of
\(P_{G_n^\top A^\top}\) and \(G_n\).
Here \(P_{G_n^\top A^\top}\) denotes the orthogonal projector onto
\(\operatorname{Range}(G_n^\top A^\top)\).  Since \(G_n=\Sigma\) is
diagonal, \(G_n^\top A^\top=\Sigma A^\top\).

Let \(\Omega=A^\top\). Since \(A\) has i.i.d. standard Gaussian entries,
\(\Omega\) is a standard Gaussian matrix, so Lemma~\ref{lem:hmt-frob-specialized}
applies.
Thus,
\[
    P_{G_n A^\top}=P_{\Sigma\Omega}.
\]
Thus
\[
    E(n)
    =
    \mathrm E_{\Omega}
    \bigl\|
    (I_n-P_{\Sigma\Omega})\Sigma
    \bigr\|_F^2 .
\]
Applying Lemma~\ref{lem:hmt-frob-specialized} with
\(Y=\Sigma\Omega\) gives
\[
    E(n)
    \le
    \left(1+\frac{k}{m-k-1}\right)
    \|\Sigma_2\|_F^2 .
\]
Since
\[
    \|\Sigma_2\|_F^2
    =
    \sum_{i=k+1}^n\sigma_i^2,
\]
we obtain
\[
    E(n)
    \le
    \left(1+\frac{k}{m-k-1}\right)
    \sum_{i=k+1}^n\sigma_i^2
    =
    E(k),
\]
where the last equality is Eq.\eqref{eq:exact-Ek-overdetermined} This
proves both the closed-form formula for \(E(k)\) and the comparison
\(E(n)\le E(k)\).
\end{proof}}
We next record an auxiliary estimate used in the analysis of the overdetermined
regime \(k<m\). In this regime, the reconstruction error contains the term $T(k)
    :=
    \mathrm E_{A_k}
    \left\|
        G_k (A_kG_k)^\dagger
    \right\|_F^2,$
which measures the size of the least-squares reconstruction operator associated
with the \(k\)-dimensional prior. The following lemma gives an explicit formula
for this term.
\begin{lemma}\label{lemma9}
Under Lemma \ref{lemma8} assumptions, let $T(k) =\mathrm E_{A_k} \Big\Arrowvert G_k\big(A_kG_k\big)^{\dagger}\Big\Arrowvert_F^2 $. For all $k+1 \leq m$, $T(k) =\frac{k}{m-k-1}$, which is monotone increasing in \(k\).
\end{lemma}
\begin{proof}[Proof of Lemma \ref{lemma9}.]
To establish the desired result, recall $A_k \in R^{m\times k}, G_k \in \R^{k\times k }.$ Let $M_k = A_kG_k \in \R^{m\times k}$, for $k\leq m$, observe that $\text{Rank}(M_k)=k$, hence $M_k$ has full column rank, and $M_k^{\dagger} = (M_k^{\top}M_k)^{-1}M_k^{\top}.$   
    \allowdisplaybreaks{
We have
    \begin{align*}
       \Big\Arrowvert G_k\big(A_kG_k\big)^{\dagger}\Big\Arrowvert_F^2 &= \Big\Arrowvert G_kM_k^{\dagger}\Big\Arrowvert_F^2 \\&= \Big\langle G_k M_k^{\dagger},G_k M_k^{\dagger} \Big\rangle  \\
        & = \Big\langle G_k^TG_k,M_k^{\dagger}(M_k^{\dagger})^T\Big\rangle\\
        & = \Big\langle G_k^TG_k,(M_k^{\top}M_k)^{-1}M_k^{\top}M_k(M_k^{\top}M_k)^{-1}\Big\rangle\\
        & = \Big\langle G_k^TG_k,(M_k^{\top}M_k)^{-1}\Big\rangle\\
        & = \Big\langle G_k^TG_k,(G_k^{\top}A_k^{\top}A_kG_k)^{-1}\Big\rangle\\
        & = \Big\langle G_k^TG_k,G_k^{-1}(A_k^{\top}A_k)^{-1}G_k^{-\top}\Big\rangle\\
        & = \Big\langle I,(A_k^{\top}A_k)^{-1}\Big\rangle\\
       & = \text{Tr}\Big((A_k^{\top}A_k)^{-1}\Big)\\
\end{align*}
Therefore, following the Inverse Wishart Matrix definition, and apply Lemma~\ref{lem:wishart-inverse}, we know that $\mathrm E_{A_k}\Big\Arrowvert G_k\big(A_kG_k\big)^{\dagger}\Big\Arrowvert_F^2  =  \mathrm E_{A_k}\text{Tr}\Big((A_k^{\top}A_k)^{-1}\Big) = E_{A_k}\frac{I_k}{m-k-1} = \frac{k}{m-k-1}$, where $k < m-1$. It is clear that \(T(k)\) increases monotonically with \(k\).}
\end{proof}

\subsection{Boundary divergence at the measurement threshold}\label{sec:boundary-divergence}
In this section, we consider the two critical dimensions
\(k=m-1\) and \(k=m\). We show that
\[
E(m-1)=+\infty,\quad E(m)=+\infty.
\]
\begin{proof}[Proof of Proposition \ref{prop:boundary-divergence}.]
Recall from Section~\ref{sec:diag} that, without loss of generality,
we can take
\(G_n=\operatorname{diag}(\sigma_1,\ldots,\sigma_n),\)
and let \(A_k\in\mathbb R^{m\times k}\) denote the first \(k\) columns of
\(A\).

Fix \(k\in\{m-1,m\}\). Since \(A_k\) has full column rank almost surely and
\(G_k\) is invertible,
$(A_kG_k)^\dagger=G_k^{-1}A_k^\dagger,$
and therefore
$G_k(A_kG_k)^\dagger=A_k^\dagger.$
The reconstruction-error formula Eq.\eqref{eq:E(k)over} for \(k\le m\) consequently gives
\[
    E(k)
    =
    \left(
        1+\mathrm E\|A_k^\dagger\|_F^2
    \right)
    \sum_{i=k+1}^n\sigma_i^2.
\]
Since \(k<n\) and
$\sum_{i=k+1}^n\sigma_i^2>0.$
It therefore remains to prove that
$\mathrm E\|A_k^\dagger\|_F^2=+\infty.$

Let
$ A_k=Q_kR_k$
be the reduced QR decomposition, where
\(Q_k\in\mathbb R^{m\times k}\) has orthonormal columns and
\(R_k\in\mathbb R^{k\times k}\) is upper triangular with positive diagonal
entries. Since
$A_k^\dagger=R_k^{-1}Q_k^\top,$
we have
\[
    \|A_k^\dagger\|_F^2
    =
    \|R_k^{-1}Q_k^\top\|_F^2
    =\|R_k^{-1}\|_F^2
=
\sum_{i,j=1}^k |(R_k^{-1})_{ij}|^2
\ge
|(R_k^{-1})_{k,k}|^2
=
\frac{1}{(R_k)_{k,k}^2}.
\]
The second equality follows from $Q_k^\top Q_k=I_k$; equivalently, right multiplication by $Q_k^\top$ preserves the Frobenius norm in this case. The last equality follows from the fact that $R_k$ is invertible and upper triangular, so the diagonal entries of $R_k^{-1}$ are the reciprocals of those of $R_k$; in particular, $(R_k^{-1})_{kk}=(R_k)_{kk}^{-1}$.

By the Bartlett decomposition \cite[Theorem~3.2.14]{muirhead2009aspects} for a standard Gaussian
\(m\times k\) matrix,
\[
   (R_k)_{j,j}^2\sim\chi^2_{m-j+1},
    \qquad j=1,\ldots,k.
\]
In particular,
\[
    (R_k)_{k,k}^2\sim\chi^2_{m-k+1}.
\]
For \(k\in\{m-1,m\}\), the number of degrees of freedom
\(m-k+1\) is either \(2\) or \(1\). Since
\[
    \mathrm E\left[\frac{1}{\chi_\nu^2}\right]=+\infty,
    \qquad \nu\le 2,
\]
it follows that
\[
    \mathrm E\|A_k^\dagger\|_F^2
    \ge
    \mathrm E\left[\frac{1}{(R_k)_{k,k}^2}\right]
    =
    +\infty.
\]
Consequently,
\[
    E(k)=+\infty,
    \qquad k\in\{m-1,m\},
\]
which proves the claim.
\end{proof}

\section*{Acknowledgments}
The authors thank Sean Gunn for valuable feedback and discussions. 

\section*{Funding}
PH acknowledges support from NSF Awards DMS-1848087 and DMS-2022205.
\section*{Data Availability}
No experimental datasets were generated or analyzed in this study. The numerical experiments reported in the paper use synthetically generated data according to the models described in the manuscript.

\clearpage
\appendix
\section{MAP interpretation of the reconstruction rule}\label{sec:map}
We briefly explain how the estimator in \eqref{eq0} arises from a MAP principle.
Let $\mathcal{G}_k \subset \mathbb{R}^n$ denote a $k$-dimensional generative model. The maximum a posteriori estimator of $x^\star$ under the prior induced by $\mathcal{G}_k$ is
\begin{align*}
    \hat{x}_{\text{MAP}}(k)=
    \underset{x \in \text{Range}(\mathcal{G}_k)}{\argmin}
 \frac{1}{2}\left\|y - Ax \right\|_2^2+ \gamma^2\log \mathcal{P}_{\mathcal{G}_k(x)},
\end{align*}
where $\mathcal{P}_{\mathcal{G}_k(x)}$ is the density function induced by $\mathcal{G}_k(x) $ over $\text{Range}\ (\mathcal{G}_k)$ and evaluated at $x$, and $\left\|\cdot \right\|_2$ is the $\ell_2$ norm, and $\gamma$ denotes the noise level.
In the
noiseless compressed sensing setting studied here, we consequently analyze the MAP estimator as $\gamma \downarrow 0$.
Suppose $A$ is full rank, the reconstruction rule
takes different forms depending on whether the latent dimension is smaller
or larger than the number of measurements.
Thus, the MAP estimator can be written as:
\begin{align*}
    \hat{x}(k) = 
\begin{cases} 
\underset{x \in \text{Range}(\mathcal{G}_k)}{\argmin} \log\mathcal{P}_{\mathcal{G}_k(x)} \quad \text{s.t.} \quad y = Ax, & \text{for } k \geq m  \\
\underset{x \in \text{Range}(\mathcal{G}_k)}{\argmin}\ \frac{1}{2} \left\| y-Ax \right\|_2^2, & \text{for }  k < m
\end{cases} \
\end{align*}
It is more convenient to work with the estimator through their latent representation. Thus, in latent coordinates, the MAP estimator
can be written as $\hat{x}(k) = \mathcal{G}_k
\hat{z}(k)$
where
\begin{align*}
\hat{z}(k) = 
\begin{cases} 
\underset{z_k \in \R^k}{\argmin}\ \frac{1}{2}\| z_k \|_2^2 \quad \text{s.t.} \quad y = A\mathcal{G}_k z_k, & \text{for } k \geq m  \\
\underset{z_k \in \R^k}{\argmin} \ \frac{1}{2} \left\| A\mathcal{G}_k z_k - y \right\|_2^2, & \text{for }  k < m
\end{cases} \  
\end{align*}
This yields the two-regime estimator stated in \eqref{eq0}.
\section{Proofs of auxiliary results}
\begin{lemma}\label{Lemma10}
Let $x \sim \mathcal{N}\left(0, I_n\right) \in \R^n$ then for any matrix $M \in \R^{m \times n}$, we have
    $$
    \mathrm{E}_x\|M x\|^2=\|M\|_F^2 .
    $$
\end{lemma} 
\begin{proof}[Proof of Lemma \ref{Lemma10}.]
$$
\mathbb{E}_x\|M x\|^2=\mathbb{E}_x\langle M x, M x\rangle=\mathbb{E}_x\left\langle M^T M, x x^T\right\rangle=\left\langle M^T M, I_n\right\rangle=\|M\|_F^2 .
$$
\end{proof}
\begin{lemma}\label{Lemma11}
    Let $x \sim \mathcal{N}(0, \Sigma) \in \R^n$ then for any matrix $M \in \R^{m \times n}$, we have
    $$
    \mathrm{E}_x\|M x\|^2=\operatorname{Tr}\left(M \Sigma M^T\right) .
    $$
\end{lemma}
\begin{proof}[Proof of Lemma \ref{Lemma11}.] 
    Let $Y=M x$, then $Y \sim \mathcal{N}\left(0, M \Sigma M^T\right)$. Thus,
$$
\mathbb{E}_x\|M x\|^2=\mathbb{E}_Y\|Y\|^2=\mathbb{E}\left[Y^T Y\right]=\sum_i^n Y_i^2=\sum_{i=1}^n \operatorname{Var}\left(Y_i\right)=\operatorname{Tr}\left(M \Sigma M^T\right).
$$
\end{proof}
\begin{lemma}\cite{vershynin2010introduction}
Let $X, Y$ be independent isotropic random vectors in $\mathbb{R}^n$. Then $\mathbb{E}\|X\|_2^2=$ $n$ and $\mathbb{E}\langle X, Y\rangle^2=n$.
\end{lemma}
\begin{lemma}\label{lemma13}
    Let $M\in \R^{m\times m}\succeq 0,\widetilde{M}\in \R^{m\times m} \succeq 0$ be positive semidefinite (PSD) matrix, 
    $\widetilde{M} $ is the perturbation of $M$, we have 
    \begin{align*}
    \big(M+\widetilde{M}\big)^{-1} - M^{-1}+M^{-1}\widetilde{M}M^{-1}  \succcurlyeq 0 
    \end{align*}
\end{lemma}
\begin{proof}[Proof of Lemma \ref{lemma13}.] 
\begin{align*}
\Big(M+\widetilde{M}\Big)^{-1}& 
 = \Big[M^{\frac{1}{2}}\Big(I+M^{-\frac{1}{2}}\widetilde{M}M^{-\frac{1}{2}}\Big)M^{\frac{1}{2}}\Big]^{-1}\\
 & =M^{-\frac{1}{2}}\Big(I+M^{-\frac{1}{2}}\widetilde{M}M^{-\frac{1}{2}}\Big)^{-1}M^{-\frac{1}{2}}
\end{align*}
\begin{align*}
    \Big(M+\widetilde{M}\Big)^{-1}-M^{-1}&= M^{-\frac{1}{2}}\Big(I+M^{-\frac{1}{2}}\widetilde{M}M^{-\frac{1}{2}}\Big)^{-1}M^{-\frac{1}{2}}-M^{-1}\\
    & = M^{-\frac{1}{2}}\bigg[\Big(I+M^{-\frac{1}{2}}\widetilde{M}M^{-\frac{1}{2}}\Big)^{-1}-I\bigg]M^{-\frac{1}{2}}
\end{align*}
Thus,
\begin{align*}
\Big(M+\widetilde{M}\Big)^{-1}-\Big(M^{-1} - M^{-1}\widetilde{M}M^{-1}\Big)& = M^{-\frac{1}{2}}\Big[\Big(I+M^{-\frac{1}{2}}\widetilde{M}M^{-\frac{1}{2}}\Big)^{-1}-I\Big]M^{-\frac{1}{2}} + M^{-1}\widetilde{M}M^{-1}\\
   & = M^{-\frac{1}{2}}\Big[\Big(I+M^{-\frac{1}{2}}\widetilde{M}M^{-\frac{1}{2}}\Big)^{-1}-\Big(I-M^{-\frac{1}{2}}\widetilde{M}M^{-\frac{1}{2}}\Big)\Big]M^{-\frac{1}{2}}
\end{align*} 
Since $M$ and $\widetilde{M}$ is PSD, and $M^{-\frac{1}{2}}\widetilde{M}M^{-\frac{1}{2}}$ hands a basis of eigenvectors, in that basis, everything is diagonalible. \\
Let's $M^{-\frac{1}{2}}\widetilde{M}M^{-\frac{1}{2}}  = U\Lambda^2U^T$ \\
Expand  $(I+M^{-\frac{1}{2}}\widetilde{M}M^{-\frac{1}{2}})^{-1}$ in the basis given by the $U$
\begin{align*}
    (I+M^{-\frac{1}{2}}\widetilde{M}M^{-\frac{1}{2}})^{-1} = (I+U\Lambda^2U^T)^{-1} =U(I+\Lambda^2)^{-1}U^{T}
\end{align*}
Similarly, expand this linear term $I- M^{-\frac{1}{2}}\widetilde{M}M^{-\frac{1}{2}}$ in the basis given by $U$
\begin{align*}
    I- M^{-\frac{1}{2}}\widetilde{M}M^{-\frac{1}{2}} =I - U\Lambda^2U^T =  U(I-\Lambda^2)^{-1}U^{T}
\end{align*}
\begin{align*}
   \Big(I+M^{-\frac{1}{2}}\widetilde{M}M^{-\frac{1}{2}}\Big)^{-1}-\Big(I-M^{-\frac{1}{2}}\widetilde{M}M^{-\frac{1}{2}}\Big)&=U\Big(I + \Lambda^2\Big)^{-1}U^T -U\Big(I-\Lambda^2 \Big)U^T\\
   & =U\bigg[\Big(I + \Lambda^2\Big)^{-1}-\Big(I-\Lambda^2 \Big)\bigg]U^T
\end{align*}
Since we know that $\frac{1}{x}$ is a convex function for $x > 0$, and $1-x$ is the linear approximation around $x= 1$, so we get:
\begin{align*}
    (1+\lambda_i^2)^{-1}
-(1-\lambda_i^2) \geq 0\end{align*}
So we know that $\Big(I+M^{-\frac{1}{2}}\widetilde{M}M^{-\frac{1}{2}}\Big)^{-1}-\Big(I-M^{-\frac{1}{2}}\widetilde{M}M^{-\frac{1}{2}}\Big)$ is PSD.
\end{proof}
\begin{claim}\label{Claim4.1}
    Let $H(k) = E_{\widetilde{A_{k-1}}}\Big[h(A_{k},t)\Big].$ $\forall A_{k},h(A_{k},t)$ is monotone. It is sufficient to show that the monotonicity of $H(k)$ is determined by  $h(A_{k},t)$.
\end{claim}
\begin{proof}[Proof of Claim \ref{Claim4.1}.]
    \begin{align*}
       D &=  H(k)-H(k-1)\\ & = E_{\widetilde{A_{k-1}}}\Big[h(A_{k},t)\Big] - E_{\widetilde{A_{k-1}}}\Big[h(A_{k-1},t)\Big]\\
        & = E_{\widetilde{A_{k-1}}}\Big[h(A_{k},t)- h(A_{k-1},t)\Big]
    \end{align*}
Therefore, we know that the sign of $D$ is decided by the sign of $h(A_{k},t)- h(A_{k-1},t)$.
\end{proof}
\begin{theorem}[Deterministic error bound \cite{halko2011finding}]\label{deterministic error bound}
    Let $A$ be an $m \times n$ matrix with SVD $A=U\Sigma V$, and fix $k \geq 0$. Choose a test matrix $\Omega$, and construct the sample matrix $Y=A \Omega$. Partition $\Sigma$ as 

\[
A = U
\begin{bmatrix}
\Sigma_1 & \\
& \Sigma_2
\end{bmatrix}
\begin{bmatrix}
V_1^{*} \\
V_2^{*}
\end{bmatrix},
\qquad
\Sigma_1 \in \mathbb{R}^{k \times k},\;
\Sigma_2 \in \mathbb{R}^{(n-k)\times(n-k)}.
\]
and define $\Omega_1$ and $\Omega_2$ via $\Omega_1=V_1^\star  \Omega$ and $ \Omega_2=V_2^\star  \Omega$. Assuming that $\Omega_1$ has full row rank, the approximation error satisfies
\begin{equation}
    \left\|\left(\mathbf{I}-P_{Y}\right) A\right\|_F^2 \leq\left\|\Sigma_2\right\|_F^2+ \| \Sigma_2 \Omega_2 \Omega_1^{\dagger} \|_F^2,
\end{equation}
\end{theorem}
\begin{proof}[Proof of Theorem \ref{deterministic error bound}.]
Throughout, write $\Sigma=\operatorname{diag}(\Sigma_1,\Sigma_2)$ and recall
\[
\tilde{A}:=U^\star A=\Sigma V^\star 
=\begin{bmatrix}\Sigma_1 V_1^\star \\[2pt]\Sigma_2 V_2^\star \end{bmatrix},
\qquad
\tilde{Y}:=\tilde{A}\Omega
=\begin{bmatrix}\Sigma_1\Omega_1\\[2pt]\Sigma_2\Omega_2\end{bmatrix},
\]
where the block sizes are $k$ and $n-k$.

We note that the left factor $U$ is irrelevant.
Because $\tilde{Y}=U^\star Y$ and orthogonal projectors transform by conjugation, $P_{\tilde{Y}}=U^\star P_{Y} U$. Hence
\[
(\mathbf I- P_{\tilde{Y}})\tilde{A}
=U^\star (\mathbf I- P_{Y})A,
\]
and since the Frobenius norm is unitarily invariant,
\[
\bigl\|(\mathbf I- P_{Y})A\bigr\|_F
=\bigl\|(\mathbf I- P_{\tilde{Y}})\tilde{A}\bigr\|_F .
\]
It therefore suffices to bound $\|(\mathbf I- P_{\tilde{Y}})\tilde{A}\|_F^2$.

We assume $k$ is chosen so that every diagonal entry of $\Sigma_1$ is strictly positive. If not, then $\Sigma_2=\mathbf 0$ by the ordering of the singular values, and both blocks of $\tilde{A}$ lie in the row space of $\Omega_1$ (which has full row rank), so
\[
\operatorname{range}(\tilde{A})
=\operatorname{range}\!\begin{bmatrix}\Sigma_1 V_1^\star \\ \mathbf 0\end{bmatrix}
=\operatorname{range}\!\begin{bmatrix}\Sigma_1\Omega_1\\ \mathbf 0\end{bmatrix}
=\operatorname{range}(\tilde{Y}).
\]
In that case $(\mathbf I-P_{\tilde{Y}})\tilde{A}=\mathbf 0$ and the bound holds trivially (both sides vanish). So assume $\Sigma_1$ is invertible.

The idea is to replace $\tilde{Y}$ by a matrix whose top block is the identity, obtained by ``flattening'' the top spectrum. Since $\Omega_1$ has full row rank we have $\Omega_1\Omega_1^{\dagger}=\mathbf I$, so define
\[
Z:=\tilde{Y}\,\Omega_1^{\dagger}\Sigma_1^{-1}
=\begin{bmatrix}\mathbf I\\[2pt] F\end{bmatrix},
\qquad
F:=\Sigma_2\Omega_2\Omega_1^{\dagger}\Sigma_1^{-1}.
\]
(The top block is $\Sigma_1\Omega_1\Omega_1^{\dagger}\Sigma_1^{-1}=\mathbf I$; the bottom block is $F$.)
Because $Z$ is $\tilde{Y}$ times a matrix on the right, $\operatorname{range}( Z)\subseteq\operatorname{range}(\tilde{Y})$. Enlarging the range can only decrease the residual of the associated projector, so
\begin{equation}
\bigl\|(\mathbf I-P_{\tilde{Y}})\tilde{A}\bigr\|_F
\;\le\;
\bigl\|(\mathbf I- P_{Z})\tilde{A}\bigr\|_F .
\label{eq:range-mono}
\end{equation}

\medskip
The next step is to find the explicit form of the complementary projector $\mathbf I-P_{Z}$.
The matrix $Z$ has full column rank, so $P_{Z}= Z(Z^\star Z)^{-1}Z^\star $ with $Z^\star Z=\mathbf I+F^\star F$. A direct block computation gives
\[
\mathbf I- P_{Z}
=\begin{bmatrix}
\mathbf I-(\mathbf I+ F^\star F)^{-1} & -(\mathbf I+ F^\star F)^{-1}F^\star \\[4pt]
-F(\mathbf I+F^\star F)^{-1} & \mathbf I-F(\mathbf I+F^\star F)^{-1}F^\star 
\end{bmatrix},
\]
whose block structure conforms with the partition of $\Sigma$. Write $B:=-(\mathbf I+F^\star F)^{-1}F^\star $ for the off-diagonal block.

\medskip
Now, we can find the semidefinite bound on the projector.
Set $M:=F^\star  F\succeq\mathbf 0$. For the diagonal blocks:
\begin{itemize}
\item Top-left: for any psd $M$ one has $\mathbf I-(\mathbf I+ M)^{-1}\preccurlyeq M$ (equivalently $(\mathbf I+ M)^{-1}\succeq\mathbf I-M$, which follows entrywise on eigenvalues from $\tfrac{1}{1+\mu}-(1-\mu)=\tfrac{\mu^2}{1+\mu}\ge0$). Thus the top-left block is $\preccurlyeq F^\star F$.
\item Bottom-right: $F(\mathbf I+F^\star F)^{-1}F^\star \succeq\mathbf 0$, so this block is $\preccurlyeq\mathbf I$.
\end{itemize}
Since replacing the two diagonal blocks by these upper bounds leaves the off-diagonal blocks unchanged, the difference is block-diagonal with psd diagonal blocks, hence psd. Therefore
\[
\mathbf I-P_{Z}\;\preccurlyeq\;
\begin{bmatrix}F^\star F &  B\\[2pt]  B^\star  & \mathbf I\end{bmatrix}.
\]

Congruence preserves the semidefinite order, so conjugating by $\Sigma=\operatorname{diag}(\Sigma_1,\Sigma_2)$ gives
\[
\Sigma^\star (\mathbf I-P_{Z})\Sigma
\;\preccurlyeq\;
\begin{bmatrix}
\Sigma_1^\star  F^\star  F\Sigma_1 & \Sigma_1^\star  B\Sigma_2\\[3pt]
\Sigma_2^\star  B^\star \Sigma_1 & \Sigma_2^\star \Sigma_2
\end{bmatrix}.
\]
Now use $\tilde{A}=\Sigma V^\star $ together with unitary invariance of the trace:
\[
\bigl\|(\mathbf I-P_{Z})\tilde{A}\bigr\|_F^2
=\operatorname{Tr}\!\bigl(\tilde{A}^\star (\mathbf I- P_{Z})\tilde{A}\bigr)
=\operatorname{Tr}\!\bigl(\Sigma^\star (\mathbf I-P_{ Z})\Sigma\bigr).
\]
The trace is monotone with respect to $\preccurlyeq$ and picks out only the diagonal blocks, so
\[
\operatorname{Tr}\!\bigl(\Sigma^\star (\mathbf I-P_{Z})\Sigma\bigr)
\;\le\;
\operatorname{Tr}\!\bigl(\Sigma_1^\star  F^\star  F\Sigma_1\bigr)
+\operatorname{Tr}\!\bigl(\Sigma_2^\star \Sigma_2\bigr)
=\bigl\|F\Sigma_1\bigr\|_F^2+\bigl\|\Sigma_2\bigr\|_F^2 .
\]

Finally, by definition of $F$,
\[
 F\Sigma_1
=\Sigma_2\Omega_2\Omega_1^{\dagger}\Sigma_1^{-1}\Sigma_1
=\Sigma_2\Omega_2\Omega_1^{\dagger},
\]
so $\|F\Sigma_1\|_F^2=\|\Sigma_2\Omega_2\Omega_1^{\dagger}\|_F^2$. Combining all together,
\[
\bigl\|(\mathbf I-P_{Y})A\bigr\|_F^2
=\bigl\|(\mathbf I-P_{\tilde{Y}})\tilde{A}\bigr\|_F^2
\le\bigl\|(\mathbf I-P_{Z})\tilde{A}\bigr\|_F^2
\le\bigl\|\Sigma_2\bigr\|_F^2+\bigl\|\Sigma_2\Omega_2\Omega_1^{\dagger}\bigr\|_F^2,
\]
which is the asserted bound.
\end{proof}


\section{Wishart and Inverse Wishart Distributions}
In this appendix we collect the basic definitions and facts on the Wishart and Inverse Wishart distributions that are used in the main text. Proofs of these results can be found in standard references on multivariate statistics (e.g., Muirhead, 1982).

\subsection{Wishart distribution}

\begin{definition}[Wishart distribution]
Let $X\in\mathbb{R}^{n\times p}$ have i.i.d. rows $x_i\sim\mathcal N(0,\Sigma)$, where 
$\Sigma\in\mathbb{R}^{p\times p}$ is positive definite.  
The random matrix
\[
W = X^\top X = \sum_{i=1}^n x_i x_i^\top
\]
is said to follow a \emph{Wishart distribution} with $n$ degrees of freedom and scale matrix $\Sigma$, written
\[
W \sim \mathcal W_p(\Sigma,n).
\]
\end{definition}

\begin{lemma}[Basic Properties of Wishart Matrices]\label{lem:wishart-basic}
If $W\sim \mathcal W_p(\Sigma,n)$, then:
\begin{enumerate}
    \item $W$ is almost surely symmetric positive semidefinite.
    \item $\mathbb{E}[W]=n\Sigma$.
    \item If $n\ge p$, then $W$ is almost surely invertible.
\end{enumerate}
\end{lemma}

\begin{remark}
When $\Sigma=I_p$, $W$ is the sample covariance matrix (up to a factor of $n$) for standard Gaussian data.
\end{remark}

\subsection{Inverse Wishart distribution}

\begin{definition}[Inverse Wishart distribution]
A random matrix $S\in\mathbb{R}^{p\times p}$ is said to follow an \emph{Inverse Wishart distribution} with scale matrix $\Psi$ and $n$ degrees of freedom, written
\[
S \sim \mathcal W_p^{-1}(\Psi,n),
\]
if $S=W^{-1}$ where $W\sim \mathcal W_p(\Psi^{-1},n)$.
\end{definition}

\begin{lemma}[Moments of the inverse Wishart]\label{lem:iw-moment}
If $S\sim \mathcal W_p^{-1}(\Psi,n)$ and $n>p+1$, then
\[
\mathbb{E}[S]=\frac{\Psi}{\,n-p-1\,}.
\]
\end{lemma}

\begin{remark}
The Inverse Wishart distribution is commonly used as a conjugate prior for covariance matrices in Bayesian statistics.
\end{remark}

\subsection{Identities used in the paper}
The following identities are used repeatedly in Sections 3.4.
\begin{lemma}[Useful expectations]\label{lem:wishart-inverse}
If $W\sim \mathcal W_p(\Sigma,n)$ with $n>p+1$, then
\[
\mathbb{E}[W^{-1}] 
= \frac{1}{n-p-1}\Sigma^{-1},
\qquad
\mathbb{E}\big[\mathrm{Tr}(W^{-1})\big]
= \frac{p}{n-p-1}\,\mathrm{Tr}(\Sigma^{-1}).
\]
\end{lemma}

\begin{lemma}[Concentration]\label{lem:wishart-lln}
As $n\to\infty$,
\[
\frac{1}{n}W \;\xrightarrow{\;a.s.\;}\; \Sigma,
\]
that is, the Wishart law concentrates around $n\Sigma$.
\end{lemma}
\section{Supplementary experiment result from Section \ref{sec:standing}}
\begin{figure}[H]
    \centering
    \includegraphics[width=0.75\linewidth]{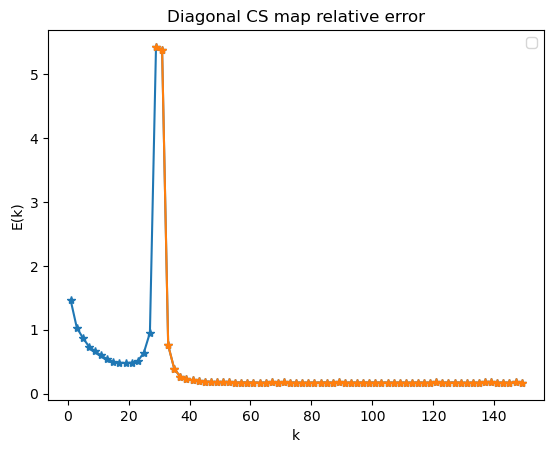}
    \caption{Relative reconstruction error in compressed sensing with diagonal generator matrix. It shows reconstruction error versus model dimension k in compressed sensing with $n=150, m=30$. The generator matrix $G_n $ has diagonal structure with exponentially decaying singular values. The error exhibits three regimes predicted by our theory: a U-shaped behavior for $k<m$, a sharp peak at the interpolation threshold $k=m$, and monotone decay for $k>m$ .}
    \textbf{Alt text:}
    \textbf{Alt text:}
    A line plot of relative reconstruction error versus latent dimension. The horizontal axis shows the latent dimension, and the vertical axis shows the relative reconstruction error. The plot contains data points connected by lines across the full range of latent dimensions.
    \label{fig:k=mblows up}
\end{figure}
\end{document}